\documentclass[a4paper,fleqn]{cas-sc}

\usepackage[authoryear,longnamesfirst]{natbib}
\def\tsc#1{\csdef{#1}{\textsc{\lowercase{#1}}\xspace}}
\tsc{WGM}
\tsc{QE}
\tsc{EP}
\tsc{PMS}
\tsc{BEC}
\tsc{DE}
\DeclareUnicodeCharacter{0301}{\'{e}}
\UseRawInputEncoding

\usepackage[utf8]{inputenc}
\usepackage[T1]{fontenc}
\usepackage{textcomp}
\usepackage{algorithm}
\usepackage{algpseudocode}
\usepackage{amsmath}

\usepackage{amssymb}
\usepackage{apalike}
\usepackage{longtable}
\usepackage{array}
\usepackage{colortbl}
\usepackage{amsmath}
\usepackage{graphicx}
\usepackage[inline]{enumitem}
\usepackage{ragged2e}
\usepackage{threeparttable}
\hypersetup{
colorlinks   = true,
urlcolor     = purple,
linkcolor    = purple,
citecolor   = purple}
\makeatletter
    \g@addto@macro{\UrlBreaks}{\do\/\do\-\do\_} 
\makeatother
\usepackage[ocgcolorlinks]{ocgx2}
\usepackage{multirow}
\usepackage{float}
\usepackage{placeins}
\usepackage{afterpage}
\usepackage{appendix}
\usepackage{enumitem}
\usepackage{subcaption} % Include the subcaption package for subfigures
\usepackage[most]{tcolorbox}
\colorlet{shadecolor}{yellow}
\usepackage{mdwmath}
\usepackage{mdwtab}
\usepackage{eqparbox}
\usepackage{booktabs}
\usepackage{multirow}
\usepackage{numprint}
\definecolor{lightblue}{rgb}{0.92, 0.95, 1}
\definecolor{lightred}{rgb}{1, 0.9, 0.9}
\definecolor{deepblue}{rgb}{0, 0.4470, 0.7410}
\definecolor{deepyellow}{rgb}{0.9290, 0.6940, 0.1250}
\definecolor{deepgreen}{rgb}{0,0.5,0}
\definecolor{mygreen}{rgb}{0.01, 0.5, 0.01}
\definecolor{myred}{rgb}{0.8, 0.01, 0.01}
\usepackage{longtable}
\usepackage{url}
\usepackage{soul}

\newdefinition{rmk}{Remark}
\newtheorem{proof}{Proof}
\newproof{intuition}{intuition}
\newtheorem{proposition}{Proposition}
\newproof{pot}{Proof of Theorem \ref{thm}}

\newtheorem{definition}{Definition}
\newtheorem{remark}{Remark}

\begin{document}
\let\WriteBookmarks\relax
\def\floatpagepagefraction{1}
\def\textpagefraction{.001}

\shorttitle{Quantum-Informed Contrastive Learning with Dynamic Mixup Augmentation}

\shortauthors{Md Abrar Jahin et~al.}

\title [mode = title]{Quantum-Informed Contrastive Learning with Dynamic Mixup Augmentation for Class-Imbalanced Expert Systems}


\author[1]{Md Abrar Jahin}[orcid=0000-0002-1623-3859]
% abrar.jahin.2652@gmail.com
\ead{abrar.jahin.2652@gmail.com, jahin@usc.edu}
\cormark[1]
\credit{Conceptualization, Data curation, Formal analysis, Investigation, Methodology, Software, Writing -- original draft, Visualization}

\author[2]{Adiba Abid}[orcid=0009-0004-7443-4669]
\ead{abid1911037@stud.kuet.ac.bd, adibaabid55@gmail.com}
\credit{Data curation, Formal analysis, Methodology, Software, Writing -- original draft, Visualization}

% \author[2]{Md. Rafiquzzaman}[orcid=0000-0002-6173-7773]
% \ead{rafiq123@iem.kuet.ac.bd}
% \cormark[1]
% \credit{Supervision, Validation, Writing -- review \& editing}

\author[3]{M. F. Mridha}[orcid=0000-0001-5738-1631]
\ead{firoz.mridha@aiub.edu}
\cormark[1]
\credit{Supervision, Validation, Writing -- review \& editing}

\affiliation[1]{organization={Thomas Lord Department of Computer Science, Viterbi School of Engineering, University of Southern California},
    city={Los Angeles},
    state={CA},
    postcode={90089},
    country={USA}}

\affiliation[2]{organization={Department of Industrial Engineering and Management, Khulna University of Engineering \& Technology (KUET)},
    city={Khulna},
    postcode={9203},
    country={Bangladesh}}

\affiliation[3]{organization={Department of Computer Science, American International University-Bangladesh (AIUB)},
    city={Dhaka},
    postcode={1229},
    country={Bangladesh}}

\cortext[1]{Corresponding author(s)}

\begin{abstract}
Expert systems often operate in domains characterized by class-imbalanced tabular data, where detecting rare but critical instances is essential for safety and reliability. While conventional approaches, such as cost-sensitive learning, oversampling, and graph neural networks, provide partial solutions, they suffer from drawbacks like overfitting, label noise, and poor generalization in low-density regions. To address these challenges, we propose QCL-MixNet, a novel Quantum-Informed Contrastive Learning framework augmented with k-nearest neighbor (kNN) guided dynamic mixup for robust classification under imbalance. QCL-MixNet integrates three core innovations: (i) a Quantum Entanglement-inspired layer that models complex feature interactions through sinusoidal transformations and gated attention, (ii) a sample-aware mixup strategy that adaptively interpolates feature representations of semantically similar instances to enhance minority class representation, and (iii) a hybrid loss function that unifies focal reweighting, supervised contrastive learning, triplet margin loss, and variance regularization to improve both intra-class compactness and inter-class separability. Extensive experiments on 18 real-world imbalanced datasets (binary and multi-class) demonstrate that QCL-MixNet consistently outperforms 20 state-of-the-art machine learning, deep learning, and GNN-based baselines in macro-F1 and recall, often by substantial margins. Ablation studies further validate the critical role of each architectural component. Our results establish QCL-MixNet as a new benchmark for tabular imbalance handling in expert systems. Theoretical analyses reinforce its expressiveness, generalization, and optimization robustness.
\end{abstract}

% Keywords

\maketitle

\section{Introduction}
\label{sec:introduction}
%{Background or significance of the problem to be addressed (1 para) / Context: Challenges of imbalanced data in expert systems (e.g., medical diagnosis, fraud detection)}Motivation/the gaps your research is gonna address (1-3 para) Limited exploration of quantum-inspired architectures for imbalance learning; lack of systematic benchmarks comparing GNNs, transformers, and classical models.
%Pitch/research proposal
%Contributions:
%Novel integration of quantum entanglement principles with contrastive learning.
%Dynamic mixup augmentation guided by k-nearest neighbors.
%First large-scale benchmark covering 7 tabular datasets, 5 GNN architectures, and 15+ baselines.
%Paper organization

Expert systems (ES), AI-based decision-making frameworks, are widely used in critical areas such as medical diagnosis, fraud detection, manufacturing, cyber security, and risk analysis, where decisions must be accurate for safety and reliability \citep{shu-hsien_liao_expert_2005}. A major challenge in ES for critical applications is that the data are naturally highly imbalanced \citep{yang_rethinking_2020,wei_effective_2013,rao_data_2006}. This is a common and challenging issue to solve \citep{branco_survey_2016}. In addition to binary classification, class imbalance becomes more severe in multi-class classification problems \citep{krawczyk_learning_2016}. Though traditional deep learning (DL) has recently advanced machine learning (ML) by partly overcoming the knowledge bottleneck that limited ML and AI for decades, it remains highly sensitive to imbalanced data distribution \citep{ghosh_class_2024,huang_deep_2020,bugnon_deep_2020,DBLP:conf/icml/YangZXK19,he_deep_2016,collobert_unified_2008,ando_deep_2017,buda_systematic_2018}. In addition, performance worsens significantly with the increasing imbalance ratio \citep{pulgar_impact_2017}. Moreover, failing to detect or predict rare but critical cases can cause higher costs, serious consequences, or sometimes irreparable damage. For example, failure to detect rare invalid transactions can result in significant financial loss and loss of customer trust. Missing early-stage cancer in a few patients (false negatives) may reduce survival chances. Missing signs of rare machine faults, like a turbine blade crack, can cause equipment failure or downtime. These issues suggest to suggest to handle the class imbalance issue carefully.

While established data-level techniques \citep{batista_study_2004,barandela_strategies_2003}, including oversampling and undersampling, along with algorithm-level methods such as cost-sensitive learning (CSL), offer foundational strategies for addressing class imbalance, each comes with trade-offs. Data-level techniques are flexible and widely used since they do not depend on the choice of classifier \citep{lopez_insight_2013}. Undersampling \citep{devi_review_2020} helps reduce bias toward the majority class but may not be suitable for small datasets. On the other hand, oversampling \citep{sharma_review_2022} increases the risk of overfitting. Similarly, CSL \citep{araf_cost-sensitive_2024} is conceptually powerful but often faces practical hurdles in accurately defining misclassification costs, which is difficult for complex real-world datasets. This highlights that simply modifying data distributions or loss functions may be inadequate for effectively addressing severe imbalance in challenging tabular settings \citep{krawczyk_cost-sensitive_2014}.

These challenges have motivated the exploration of methods that can inherently learn more discriminative features. Graph Neural Networks (GNNs) \citep{4700287,gori2005new} offer a strong approach for modeling tabular data by capturing complex dependencies often overlooked by traditional models. However, applying GNNs to tabular data can be computationally intensive and face scalability limitations, especially when large datasets or instance-wise graph construction are involved \citep{villaizan-vallelado_graph_2024,lee_analysis_2024}. 

Beyond graph-based representation learning, contrastive learning (CL) has recently shown notable success in improving generalization by structuring feature representations based on underlying patterns \citep{hu_comprehensive_2024}. A notable advancement was proposed by Tao et al. \citep{tao_supervised_2024}, who combined Supervised Contrastive Learning (SCL) with automatic tuning of the temperature parameter $\tau$, a key factor influencing performance. While this approach marks great progress, it presents challenges that limit its generalizability. Specifically, it requires dataset-specific tuning of architecture (e.g., layer size, batch size), making its application less practical across diverse scenarios. Moreover, this study did not report precision and recall separately; these metrics are critical in expert systems where the costs of false positives and false negatives differ significantly. The method also depends on complex hyperparameter tuning via Tree-structured Parzen Estimator (TPE), which increases model complexity. Additionally, their use of basic augmentation, especially Gaussian blur, leaves room for more effective techniques. In fact, CL’s success in image domains is largely driven by augmentation techniques, which do not translate well to tabular data due to the lack of spatial or structural properties. 

Basic mixing of tabular samples can generate unrealistic data points, particularly harming rare class representations. Thus, more adaptive and smarter augmentation strategies are needed to generate meaningful and diverse samples without introducing much noise. Concurrently, researchers have explored new computational approaches to build neural networks that can learn complex patterns effectively, even with limited or complex data. One such direction is Quantum-inspired (QI) DL, which uses ideas from quantum mechanics to improve classical models’ learning ability. However, these QI methods are still developing, especially in handling imbalanced and complex datasets~\citep{shi_two_2023, hong_hybrid_2024, konar_quantum-inspired_2020}.

To address these key challenges in handling imbalanced tabular data, we propose \ul{Q}uantum-informed \ul{C}ontrastive \ul{L}earning with Dynamic \ul{Mix}up Augmentation \ul{Net}work (QCL-MixNet), a novel framework that combines three core innovations. Firstly, we introduce novel Quantum Entanglement-inspired Modules within our neural network architecture, which is built with the idea of quantum mechanics without directly using the quantum computing hardware. This module is used to improve the model's capacity to capture complex, non-linear feature interactions often missed by standard layers, thereby improving feature representation, offering a practical advancement for QI techniques in this domain. Secondly, to tackle the critical issue of data augmentation for tabular data, we employ a Sample-Aware Dynamic Mixup strategy. Instead of random mixing that risks generating unrealistic samples,  our method intelligently generates synthetic instances by interpolating an anchor sample with one of its k-nearest neighbors (kNN) in the feature space. This approach aims to create more realistic and beneficial augmented samples, especially for underrepresented minority classes, thus enriching the training data without introducing significant noise. Existing solutions often rely on singular approaches (e.g., data resampling, basic cost-sensitive learning, or CL with simple loss functions) that may not comprehensively address severe imbalance or the need for well-structured embeddings. To solve this issue, finally, our model components are trained using a hybrid loss function to learn robust and discriminative embeddings. This uses focal reweighting to handle imbalance, contrastive, and triplet components for structured embedding learning, and variance regularization for stable and well-separated class representations. 

Our main contributions are: 
\begin{enumerate}
    \item We propose QCL-MixNet, a novel framework that effectively integrates quantum-inspired modules for expressive feature learning, kNN-guided sample-aware dynamic mixup for intelligent augmentation, and a hybrid contrastive loss for robust imbalanced classification.
    \item We conduct extensive experiments on 18 diverse binary and multi-class imbalanced datasets. Our results show that QCL-MixNet consistently and significantly outperforms 20 state-of-the-art ML, DL, and graph-based models, establishing a new benchmark for imbalanced tabular data.
    \item We conduct systematic ablation studies to validate the critical impact of each architectural component, demonstrating that the full QCL-MixNet architecture achieves superior and stable performance.
\end{enumerate}

The remainder of this paper is organized as follows: Section \ref{sec:litreview} reviews related work on class imbalance techniques, contrastive learning, quantum-inspired deep learning, and graph-based methods for tabular data. Section \ref{sec:methodology} introduces the proposed QCL-MixNet framework, detailing its architecture and theoretical foundations. Section \ref{sec:experiments} outlines the experimental setup, including datasets, baselines, and evaluation protocol. Section \ref{sec:results} presents and discusses the results. Section \ref{sec:conclude} concludes the paper with key findings, limitations of this study, and future directions.

\section{Literature Review}
\label{sec:litreview}
\subsection{Class Imbalance Techniques (Undersampling, Oversampling, CSL)}
A training set is considered imbalanced when one class has significantly fewer samples than the others \citep{barandela_strategies_2003}. Class imbalance becomes seriously problematic when identifying rare but crucial cases. This is a widespread challenge affecting various domains such as fraud detection, software engineering, fault diagnosis, intrusion detection, network security, social media analysis, medical diagnosis, malware detection, risk assessment, solar panel fault and anomaly detection \citep{wang_learning_2019,thabtah_data_2020,yuan_review_2023,patnaik_weighted_2023,giray_use_2023,wang_lightweight_2024,dhalaria_maldetect_2024,guo_reparameterization_2025,gan_integrating_2020, DBLP:journals/tifs/ZhengHHZL25}. Conventional ML algorithms typically assume a balanced dataset, so easily affected by the imbalance issue and tend to produce biased results favoring the majority class \citep{li_meta-learning_2025}. As a result, effectively identifying minority instances in imbalanced datasets has become a key research focus \citep{dai_mutually_2025}. 

Data-level methods are the preprocessing techniques that modify the dataset itself to improve the performance of standard training procedures. Undersampling, oversampling, and hybrid fall under this criterion \citep{buda_systematic_2018}. Oversampling increases the number of minority class instances by repeating them, while undersampling reduces the majority class instances to balance the classes. Hybrid sampling combines both methods \citep{DBLP:journals/csur/ShariefISN25}. Recently, several advanced undersampling, oversampling, and hybrid approaches have been introduced. The most widely used oversampling technique, Synthetic Minority Oversampling Technique (SMOTE), generates an equal number of synthetic samples for each minority instance, which can result in class overlap \citep{tao_supervised_2024}. Recently, Simplicial SMOTE was proposed, which uses groups of nearby points to create synthetic data instead of only using pairs of points (edges) like SMOTE \citep{DBLP:conf/kdd/KachanSG25}. Isomura et al. introduced an oversampling method using large language models (LLMs) to generate more realistic and varied synthetic data for imbalanced tabular datasets \citep{isomura_llmovertab_2025}. However, there are concerns about possible bias in the synthetic data. On the undersampling side, the Schur decomposition class-overlap undersampling method (SDCU), an undersampling method, uses Schur matrix decomposition and global similarity to handle class-overlap in imbalanced datasets \citep{dai_class-overlap_2023}. Random Forest Cleaning Rule (RFCL) was also introduced to balance imbalanced data by removing overlapping majority class samples. RFCL worked well but focused only on F1-score, which restricted its adaptability~\citep{zhang_rfcl_2021}. Yu et al. introduced Balanced Training and Merging (BTM) to improve the worst-performing categories in long-tailed learning \citep{yu_reviving_2025}. Limitations included a slight decrease in arithmetic accuracy in certain scenarios. 

Despite the advancements, limitations in data-level techniques persist. Undersampling can lead to critical information loss, which is a major concern. Oversampling, on the other hand, may cause overfitting. Hybrid sampling can inherit the drawbacks of both methods, leading to information loss or noise sensitivity \citep{wang_novel_2025}. Carvalho et al. explored different data resampling methods (oversampling, undersampling, and hybrid methods, including advanced ones)~\citep{carvalho_resampling_2025}. They concluded that no single method worked best for all cases. In contrast to data-level methods, algorithm-level approaches modify the classification algorithm itself, bypassing issues of data modification. 

CSL is one of the most popular algorithm-level methods. CSL addresses class imbalance by assigning different misclassification costs to each class, typically giving higher costs to minority class errors. This approach aims to minimize costly misclassifications \citep{araf_cost-sensitive_2024}. Tang et al. proposed a robust Two-Stage instance-level CSL method with the Bounded Quadratic Type Squared Error (BQTSE) loss function that showed improved classification accuracy~\citep{tang_robust_2024}. However, it might struggle with extremely large and complex datasets. Cao et al. introduced the concept of deep imbalanced regression (DIR) for addressing the issue of imbalanced data in predicting the remaining useful life (RUL) \citep{cao_cost-sensitive_2024}. It proposed methods like label and feature distribution normalization, ranking similarity optimization, and a CSL framework to improve predictions. Nevertheless, DIR faces challenges when dealing with very small datasets and requires further development in data augmentation techniques. In summary, while significant progress has been made in addressing class imbalance through data-level techniques and algorithm-level approaches, these methods still face substantial limitations. Data-level techniques may result in information loss or overfitting, and cost-sensitive learning may struggle to determine the appropriate misclassification costs for each class, and its effectiveness depends on the specific characteristics of the dataset. These challenges indicate that while improvements have been made, a comprehensive solution is yet to be fully realized.

\subsection{CL in Non-Vision Domains}
Representation learning plays a vital role in improving the performance of the ML models by revealing the underlying factors that drive variations in the data \citep{bengio_representation_2013}. CL, a prominent approach in representation learning, aims to improve feature representation by pulling similar samples closer and pushing dissimilar ones apart \citep{zhao_cross-supervised_2025}. CL has been widely applied in various fields, including image recognition and generation, adversarial samples detection, video and graph analysis, speech recognition, natural language processing, and recommendation systems \citep{hu_comprehensive_2024}. 

CL has already demonstrated outstanding performance in computer vision tasks \citep{kottahachchi_kankanamge_don_q-supcon_2025, liu_fedcl_2023, guo_enhancing_2025, zhang_graph-weighted_2025, zhou_drtn_2025, xu_contrastive_2025, wang_scl-wc_2022}. But in non-vision domains, specifically in tabular datasets, CL is relatively less explored. Recent efforts have attempted to bridge this gap. For example, Wu et al. \citep{wu_contrastive_2023} introduced CL-enhanced Deep Neural Network with Serial Regularization (CLDNSR) to effectively handle high-dimensional data with limited samples. However, it shows limitations when applied to imbalanced datasets. Tao et al. \citep{tao_supervised_2024} applied \text{SCL-TPE}, which improved representation quality and classification accuracy for imbalanced tabular datasets. Despite its success, their approach is still limited by inadequate data augmentation strategies and sensitivity to noisy labels. These findings indicate that while CL has shown progress in handling tabular data, current methods often face limitations related to imbalance handling, label noise, and effective augmentation.

\subsection{Quantum-Inspired DL}
DL models face difficulties when dealing with very small datasets \citep{sun_revisiting_2017}. Even with larger datasets, they struggle to effectively manage highly complex, variable data. Quantum models can be a promising candidate to address some of these limitations \citep{orka_quantum_2025}.
Quantum computing is an emerging field that uses the principles of quantum mechanics and offers a potential advantage over classical computing by overcoming certain constraints \citep{rieffel_introduction_2000}.
Quantum DL (QDL) models have the potential to improve speed \citep{liu_towards_2024, saggio_experimental_2021}, parameter efficiency \citep{ciliberto_quantum_2018}, feature representation \citep{havlicek_supervised_2019, goto_universal_2021}, generalization capabilities \citep{caro_generalization_2022}, and can outperform traditional DL models by achieving higher test accuracy in certain scenarios \citep{chen_quantum_2022}. 
However, despite these advantages, quantum computing suffers from high cost, limited coherence times, sensitivity to environmental interference, and error correction issues \citep{mandal_quantum_2025, orka_quantum_2025, harrow_quantum_2017}. 

QI models don't directly use quantum computing hardware \citep{jahin_qamplifynet_2023}. Instead, they incorporate principles of quantum mechanics to improve the performance and capabilities of classical DL models. Shi et al. \citep{shi_two_2023} proposed Interpretable Complex-Valued Word Embedding (ICWE) and Convolutional Interpretable Complex-Valued Word Embedding (CICWE), two QI neural networks to improve binary text classification. However, these models still face limitations in feature extraction. Hong et al. \citep{hong_hybrid_2024}
proposed a hybrid DL model, combining convolutional neural networks (CNNs), long short-term memory (LSTM), and QI neural network (QINN) for forecasting wind speed. Konar et al. \citep{konar_quantum-inspired_2020} proposed a Quantum-Inspired Self-Supervised Network (QIS-Net) for automatically segmenting brain MRI images. While QI DL models have shown promising results across various applications, they still face several challenges, including handling imbalanced and complex datasets, computational complexity, and limited scalability. These limitations highlight the need for further advancements in this area.

\subsection{GNNs for Tabular Data}
GNNs, a specialized area within DL, offer improved performance and interpretability by effectively capturing and learning from graph-structured data \citep{tan_amogel_2025}. A key property of tabular data is that the order of features holds no significance. Likewise, in graph data, the sequence of nodes does not matter; changing their arrangement does not affect the outcomes produced by GNNs. This similarity in nature makes GNNs a strong choice for tabular datasets \citep{villaizan-vallelado_graph_2024}. In recent years, Various GNN-based models have demonstrated significant advancements in handling tabular datasets. Li et al. reviewed how GNNs were used to analyze single-cell omics data, highlighting their success in tasks like cell type identification and gene regulation \citep{10.1093/bib/bbaf109}. However, limitations included high computational costs and difficulty in capturing global data structures. The authors suggested future improvements like better scalability and integration with foundation models. 

The Multiplex Cross-Feature Interaction Network (MPCFIN) addressed feature interaction and graph connectivity challenges in tabular data using a multiplex graph structure. But it combines hand-crafted and learned structures, which may introduce redundancy or conflicting information \citep{ye_cross-feature_2024}. Villaizán-Vallelado et al. proposed Interaction Network Contextual Embedding (INCE). GNN-based contextual embeddings were applied to outperform existing DL methods on tabular data but suffered from scalability issues and high training time due to per-row graph construction and complex edge-weight learning \citep{villaizan-vallelado_graph_2024}. Lee et al. introduced an algorithm combining feature-based and similarity-based learning with GNNs and contrastive loss to improve generalization \citep{lee_analysis_2024}. However, it faced limitations in scalability, interpretability, and computational cost, some common challenges for GNNs. Collectively, these studies underscore the growing capability of GNNs to handle complex tabular structures to some extent. Yet, further innovations are required in managing tabular data class imbalance, computational cost, and dynamic augmentation, limitations that are especially critical in real-world ES.

\begin{figure*}[!ht]
    \centering
    \includegraphics[width=1\linewidth]{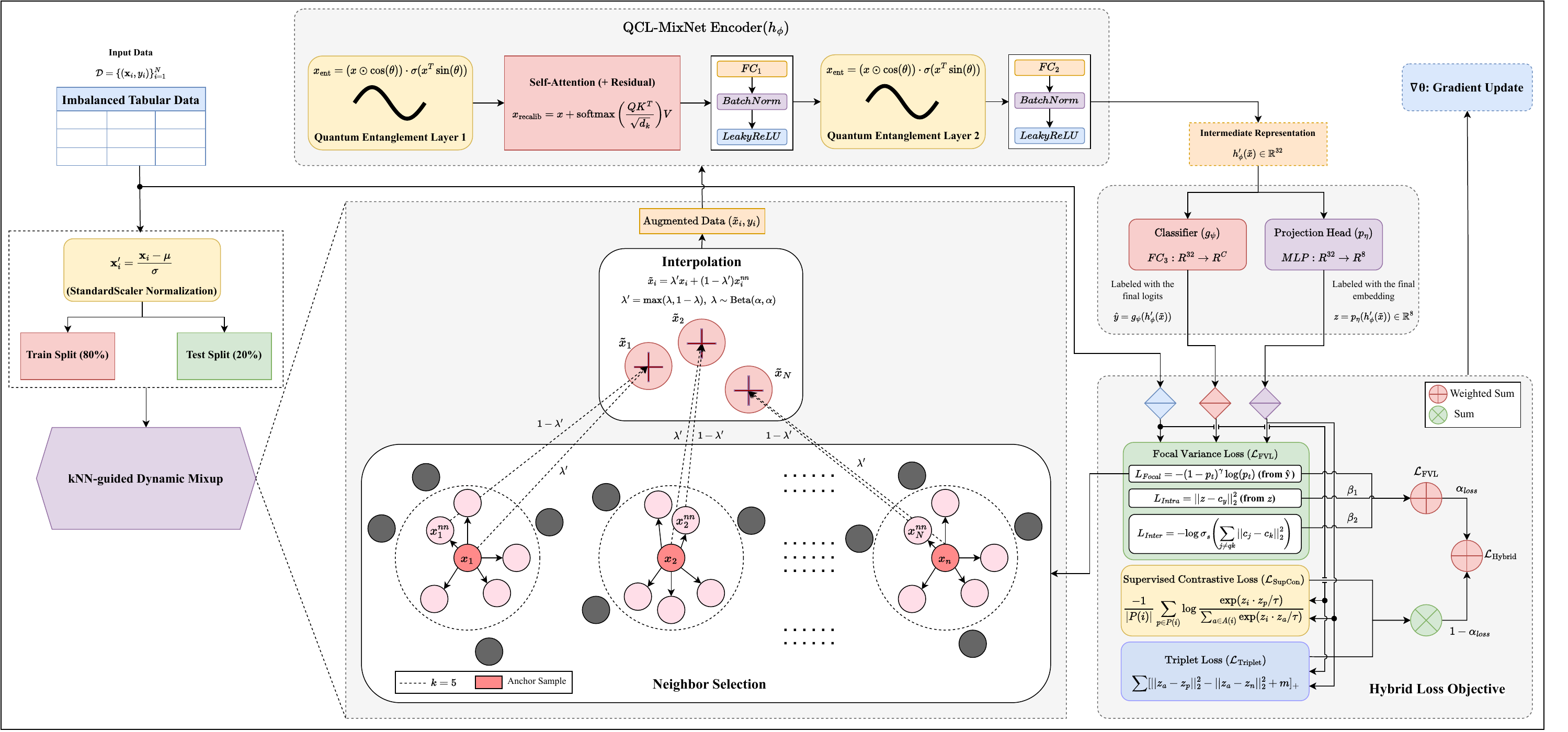}
    \caption{The overall architecture of the proposed QCL-MixNet framework for classifying imbalanced tabular data. Initially, the training data undergoes kNN-guided sample-aware dynamic mixup, where an augmented sample is generated by interpolating an anchor sample with its nearest neighbor. This augmented data is then processed by the QCL-MixNet encoder—a hybrid model featuring quantum entanglement layers and a self-attention mechanism—to produce an intermediate representation. This representation is fed into both a classifier head and a projection head. Finally, the network is optimized using a hybrid loss function that combines Focal Variance Loss, Supervised Contrastive Loss, and Triplet Loss to learn discriminative features.}
    \label{fig:framework}
\end{figure*}

\section{Materials and Methods}
\label{sec:methodology}
In this section, we outline the proposed methodology, which encompasses problem formulation, our novel QCL-MixNet architecture, a dynamic data augmentation strategy, and a hybrid loss function designed for robust representation learning and classification.

\subsection{Problem Statement}
Let $\mathcal{D} = \{(\mathbf{x}_i, y_i)\}_{i=1}^N$ be a training dataset of $N$ samples, where $\mathbf{x}_i \in \mathbb{R}^D$ is a $D$-dimensional feature vector and $y_i \in \{1, \dots, C\}$ is the corresponding class label from $C$ distinct classes. Our goal is to learn a mapping function $f_\Theta: \mathbb{R}^D \rightarrow \{1, \dots, C\}$, parameterized by $\Theta$, that accurately predicts the class label $y$ for an unseen feature vector $\mathbf{x}$. This is achieved by learning an intermediate embedding function $h_\phi: \mathbb{R}^D \rightarrow \mathbb{R}^d$ (where $d$ is the dimension of the embedding space, $d \ll D$ or $d$ can be an intermediate feature dimension) and a classifier $g_\psi: \mathbb{R}^d \rightarrow \mathbb{R}^C$, such that $f_\Theta(\mathbf{x}) = \text{argmax}(g_\psi(h_\phi(\mathbf{x})))$. The parameters $\Theta = \{\phi, \psi\}$ are optimized by minimizing a carefully designed loss function $\mathcal{L}$ over the training dataset $\mathcal{D}$.
The core of our method lies in the specific architectures for $h_\phi$ and $g_\psi$, the data augmentation techniques, and the composite nature of $\mathcal{L}$, all designed to enhance feature disentanglement, representation robustness, and classification performance, particularly in scenarios with complex data distributions.

\subsection{Model Architecture}
The proposed QCL-MixNet framework, illustrated in Figure \ref{fig:framework}, processes imbalanced tabular data through three main stages: a kNN-guided dynamic mixup for data augmentation, a quantum-informed encoder for feature representation, and a hybrid loss objective for optimization. The following subsections detail the theoretical and implementation aspects of each component.
\subsubsection{Quantum-Informed Feature Disentanglement}
To improve the model's ability to learn disentangled and informative features, we introduce quantum-informed entanglement layers augmented by attention mechanisms for feature recalibration.

\paragraph{Mathematical Formulation of Entanglement Layers.}
Inspired by the transformative operations in quantum systems, we propose a Quantum Entanglement (QE) layer. This layer is not intended to simulate quantum mechanics, but rather to leverage mathematical constructs reminiscent of quantum operations for feature transformation.
Let $\mathbf{x} \in \mathbb{R}^{d_{in}}$ be the input feature vector to the QE layer. The layer applies a set of learnable parameters $\boldsymbol{\theta} \in \mathbb{R}^{d_{in}}$. The transformation is a two-stage process:

\begin{enumerate}
    \item \textbf{Projection Stage:} The input features are first scaled element-wise, akin to a parameterized rotation or projection.
    \begin{equation}
        \mathbf{x}_{\text{proj}} = \mathbf{x} \odot \cos(\boldsymbol{\theta})
        \label{eq:qe_proj}
    \end{equation}
    where $\odot$ denotes element-wise multiplication, and $\cos(\boldsymbol{\theta})$ is applied element-wise to the parameter vector $\boldsymbol{\theta}$. This stage selectively modulates the amplitude of each feature.

    \item \textbf{Entanglement-inspired Gating Stage:} The projected features $\mathbf{x}_{\text{proj}}$ are then passed through a non-linear gating mechanism. This stage introduces interactions and dependencies across features, inspired by the concept of entanglement, where quantum states become correlated.
    A scalar gating value $s$ is computed based on $\mathbf{x}_{\text{proj}}$ and $\sin(\boldsymbol{\theta})$:
    \begin{equation}
        s = \sigma\left(\mathbf{x}_{\text{proj}}^\top \sin(\boldsymbol{\theta})\right)
        \label{eq:qe_gate_scalar}
    \end{equation}
    where $\sigma(\cdot)$ is the sigmoid activation function, ensuring $s \in (0,1)$. The final output of the QE layer $\mathbf{x}_{\text{ent}} \in \mathbb{R}^{d_{in}}$ is then:
    \begin{equation}
        \mathbf{x}_{\text{ent}} = \mathbf{x}_{\text{proj}} \cdot s
        \label{eq:qe_ent}
    \end{equation}
\end{enumerate}
The learnable parameters $\boldsymbol{\theta}$ allow the network to adaptively determine the optimal projection and feature interdependencies. The combination of cosine and sine transformations, modulated by learnable parameters, offers a rich function space for feature manipulation. The sigmoid gate allows for a soft selection or attenuation of the transformed feature set based on a collective signal derived from all features. This process aims to disentangle underlying factors of variation by creating complex, non-linear feature combinations and selectively emphasizing informative ones. Our model incorporates two such QE layers: one at the input level (acting on $\mathbb{R}^D$) and another at an intermediate feature level (acting on $\mathbb{R}^{64}$).

\paragraph{Attention Mechanisms for Feature Recalibration.}
Following the initial QE layer, we apply a self-attention mechanism to further refine and recalibrate feature representations. Self-attention allows the model to weigh the importance of different features within a sample dynamically, capturing global dependencies.
Given the output $\mathbf{x}$ from the first QE layer (or more generally, an input feature map of dimension $d_{attn}$), we treat it as a sequence of length 1 to apply attention within the sample's features (i.e., channel attention if features are considered channels). The specific implementation uses single-head attention where Query (Q), Key (K), and Value (V) matrices are derived from the same input $\mathbf{x}$.
Let $\mathbf{x} \in \mathbb{R}^{d_{attn}}$ be the input to the attention layer. It is first unsqueezed to $\mathbf{x}' \in \mathbb{R}^{1 \times d_{attn}}$ to match the expected batch-first input format for sequence length 1.
The Query, Key, and Value are computed as:
\begin{equation}
    \mathbf{Q} = \mathbf{x}' \mathbf{W}_Q, \quad \mathbf{K} = \mathbf{x}' \mathbf{W}_K, \quad \mathbf{V} = \mathbf{x}' \mathbf{W}_V
\end{equation}
where $\mathbf{W}_Q, \mathbf{W}_K, \mathbf{W}_V \in \mathbb{R}^{d_{attn} \times d_k}$ are learnable weight matrices (for a single head, $d_k = d_{attn}$). The attention output is then:
\begin{equation}
    \text{Attention}(\mathbf{Q}, \mathbf{K}, \mathbf{V}) = \text{softmax}\left(\frac{\mathbf{Q}\mathbf{K}^\top}{\sqrt{d_k}}\right)\mathbf{V}
    \label{eq:attention}
\end{equation}
The output of the attention mechanism, $\mathbf{x}_{\text{attn}}$, is added back to the input $\mathbf{x}$ via a residual connection:
\begin{equation}
    \mathbf{x}_{\text{recalibrated}} = \mathbf{x} + \mathbf{x}_{\text{attn}}
\end{equation}
This residual connection facilitates gradient flow and allows the model to adaptively decide how much recalibration is needed. By focusing on the most salient features relative to each other within a given sample, the attention mechanism complements the QE layer, improving the model's ability to extract discriminative information. In our architecture, this attention layer uses $d_{attn} = D$ (input feature dimension) and a single attention head.

\subsubsection{kNN-Guided Sample-Aware Dynamic Mixup}
Data augmentation is crucial for improving generalization. We apply a sample-aware dynamic mixup strategy, termed kNN-Guided Dynamic Mixup, which generates synthetic samples by interpolating between an anchor sample and one of its nearest neighbors in the feature space. This approach aims to create more meaningful and challenging training examples compared to the standard mixup that interpolates random pairs.

\begin{definition}[kNN-Guided Dynamic Mixup]
Let $(\mathbf{x}_i, y_i)$ be an anchor sample from a mini-batch $\mathcal{B}$.
\begin{enumerate}
    \item \textbf{Neighbor Selection:} For each $\mathbf{x}_i \in \mathcal{B}$, we identify its $k$ nearest neighbors $\{\mathbf{x}_{i,j}^{\text{NN}}\}_{j=1}^k$ from $\mathcal{B}$ (excluding $\mathbf{x}_i$ itself) based on Euclidean distance in the current feature space $h_\phi(\mathbf{x})$. One neighbor, $\mathbf{x}_{i}^{\text{NN}}$, is randomly selected from this set. Let its corresponding label be $y_{i}^{\text{NN}}$.
    \item \textbf{Interpolation Parameter Sampling:} An interpolation coefficient $\lambda$ is sampled from a Beta distribution, $\lambda \sim \text{Beta}(\alpha, \alpha)$. To ensure the anchor sample retains a dominant influence, $\lambda$ is adjusted as $\lambda' = \max(\lambda, 1-\lambda)$. This biases $\lambda'$ towards values $\geq 0.5$.
    \item \textbf{Feature Interpolation:} The mixed feature vector $\tilde{\mathbf{x}}_i$ is generated as:
    \begin{equation}
        \tilde{\mathbf{x}}_i = \lambda' \mathbf{x}_i + (1-\lambda') \mathbf{x}_{i}^{\text{NN}}
        \label{eq:knn_mixup_x}
    \end{equation}
    \item \textbf{Label Interpolation:} Labels are mixed similarly, assuming one-hot encoding $\mathbf{y}^{\text{OH}}$ for labels:
    \begin{equation}
        \tilde{\mathbf{y}}_i^{\text{OH}} = \lambda' \mathbf{y}_i^{\text{OH}} + (1-\lambda') (\mathbf{y}_{i}^{\text{NN}})^{\text{OH}}
        \label{eq:knn_mixup_y}
    \end{equation}
\end{enumerate}
The hyperparameter $\alpha$ controls the strength of interpolation, and $k=5$ in our case defines the neighborhood size.
\end{definition}

\textbf{Rationale:} Interpolating with kNNs encourages local linearity and smoothness of the decision boundary in denser regions of the feature manifold. By construction, $\mathbf{x}_{i}^{\text{NN}}$ is semantically similar to $\mathbf{x}_i$, making $\tilde{\mathbf{x}}_i$ a plausible variation. This contrasts with random-pairing mixup, which might interpolate between semantically distant samples, potentially generating less realistic examples. In our training, we use the mixed features $\tilde{\mathbf{x}}_i$, but for the classification component of our loss, we associate them with the original label $y_i$. This strategy regularizes the model to be robust to perturbations towards its neighbors while maintaining the original class identity, effectively encouraging enlargement of the decision region for class $y_i$ to include these sensible interpolations.

\subsubsection{Hybrid Contrastive Loss with Variance Regularization}
To learn discriminative and robust embeddings, we propose a hybrid loss function that integrates focal loss for classification with supervised contrastive loss, triplet loss, and an explicit variance regularization term based on learnable class centroids. Let $h_\phi(\tilde{\mathbf{x}})$ be the embedding (output of the projection head) for a (potentially augmented) input $\tilde{\mathbf{x}}$, and $g_\psi(h'_\phi(\tilde{\mathbf{x}}))$ be the raw logits from the classifier, where $h'_\phi$ is the representation before the projection head. The original label is $y$.

\paragraph{Focal Variance Loss.}
The Focal Variance Loss (FVL) component addresses class imbalance and hard example mining in classification, while simultaneously promoting intra-class compactness and inter-class separability in the embedding space.
\begin{definition}[Focal Variance Loss]
The FVL for a sample $(\tilde{\mathbf{x}}, y)$ with embedding $\mathbf{e} = h_\phi(\tilde{\mathbf{x}})$ and logits $\mathbf{z} = g_\psi(h'_\phi(\tilde{\mathbf{x}}))$ is:
\begin{equation}
    \mathcal{L}_{\text{FVL}}(\mathbf{z}, y, \mathbf{e}) = \mathcal{L}_{\text{Focal}}(\mathbf{z}, y) + \beta_1 \mathcal{L}_{\text{intra}}(\mathbf{e}, y) + \beta_2 \mathcal{L}_{\text{inter}}(y)
    \label{eq:fvl_total}
\end{equation}
where $\beta_1, \beta_2$ are weighting hyperparameters. In our implementation, $\beta_1=0.8$ and $\beta_2$ is implicitly 1 indicating $\mathcal{L}_\text{inter}$ is unweighted.
\end{definition}

\begin{enumerate}
    \item \textbf{Focal Loss Component ($\mathcal{L}_{\text{Focal}}$):} This addresses class imbalance by down-weighting the loss assigned to well-classified examples. Given the probability $p_t$ for the true class $y$ (derived from logits $\mathbf{z}$ via softmax, $p_t = \text{softmax}(\mathbf{z})_y$), the Focal Loss is:
    \begin{equation}
        \mathcal{L}_{\text{Focal}}(\mathbf{z}, y) = -(1-p_t)^\gamma \log(p_t) = (1 - p_t)^\gamma \log \frac{1}{p_t}
        \label{eq:focal_loss}
    \end{equation}
    where $\gamma \ge 0$ is the focusing parameter. We use $\gamma=3.0$ in our experiments.

    \item \textbf{Class Centroid Learning and Intra-Class Compactness ($\mathcal{L}_{\text{intra}}$):} We maintain learnable class centroids $\mathbf{c}_j \in \mathbb{R}^{d_e}$ for each class $j \in \{1, \dots, C\}$, where $d_e$ is the dimension of embeddings from the projection head ($d_e=8$ in our case). These centroids are parameters of the FVL module.
    The intra-class compactness loss penalizes the distance of an embedding $\mathbf{e}$ to its corresponding class centroid $\mathbf{c}_y$:
    \begin{equation}
        \mathcal{L}_{\text{intra}}(\mathbf{e}, y) = ||\mathbf{e} - \mathbf{c}_y||_2^2
        \label{eq:intra_variance}
    \end{equation}
    This term encourages embeddings of the same class to cluster tightly around their respective centroids. In the provided code, $\beta_1$ corresponds to `beta` in `FocalVarianceLoss`.

    \item \textbf{Inter-Class Separability ($\mathcal{L}_{\text{inter}}$):} To ensure centroids of different classes are well-separated, we introduce a penalty based on pairwise distances between centroids of classes present in the current mini-batch. Let $\mathcal{C}_{\text{batch}}$ be the set of unique classes in the current batch.
    \begin{equation}
        \mathcal{L}_{\text{inter}}(y) = -\log \sigma_s \left( \frac{1}{|\mathcal{C}_{\text{batch}}|(|\mathcal{C}_{\text{batch}}|-1)} \sum_{j \in \mathcal{C}_{\text{batch}}} \sum_{k \in \mathcal{C}_{\text{batch}}, k \neq j} ||\mathbf{c}_j - \mathbf{c}_k||_2 \right)
        \label{eq:inter_variance_conceptual}
    \end{equation}
    where $\sigma_s(\cdot)$ is the sigmoid function and $\log \sigma_s (x) = \log \frac{1}{(1 + e^{-x})}$ penalizes low centroid separation. This loss term encourages maximization of the average inter-centroid distance. 
\end{enumerate}

\begin{remark}
The use of the log sigmoid in our implementation penalizes low centroid separation. In practice, $\beta = 0.8$ controls the strength of intra-class compactness. The inter-class term is unweighted and added directly. The centroids $\{\mathbf{c}_j\}$ are learnable and updated jointly with the network.
\end{remark}
The complete FVL objective integrates all three components into the following expression:
\begin{equation}
\mathcal{L}_{\text{FVL}} = \mathbb{E}_{(\mathbf{x}, y)} \left[ \underbrace{(1 - p_t)^\gamma \log \frac{1}{p_t}}_{\text{Focal Loss}} + \underbrace{\beta_1 \cdot \|\mathbf{e} - \mathbf{c}_y\|_2^2}_{\text{Intra-Class Compactness}} \right] + \underbrace{-\log \sigma\left( \mathbb{E}_{j \ne k} \|\mathbf{c}_j - \mathbf{c}_k\|_2 \right)}_{\text{Inter-Class Separation}}
\end{equation}

\paragraph{Integration with Supervised Contrastive and Triplet Loss.}
The FVL is combined with established metric learning losses to further structure the embedding space. This forms our Hybrid Loss.
\begin{definition}[Hybrid Loss]
The total hybrid loss $\mathcal{L}_{\text{Hybrid}}$ for a mini-batch is:
\begin{equation}
    \mathcal{L}_{\text{Hybrid}} = \alpha_{\text{loss}} \mathcal{L}_{\text{FVL}} + (1-\alpha_{\text{loss}}) (\mathcal{L}_{\text{SupCon}} + \mathcal{L}_{\text{Triplet}})
    \label{eq:hybrid_loss}
\end{equation}
where $\alpha_{\text{loss}}$ is a weighting factor ($\alpha_{\text{loss}}=0.5$ in this study).
\end{definition}

\begin{enumerate}
    \item \textbf{Supervised Contrastive Loss ($\mathcal{L}_{\text{SupCon}}$):}
    This temperature-scaled loss~\citep{supconloss} encourages embeddings of samples from the same class to lie closer together in the representation space, while pushing apart embeddings from different classes. Given an anchor embedding \( \mathbf{e}_i \in \mathbb{R}^{d_e} \) with class label \( y_i \), the supervised contrastive loss is defined as:
    \begin{equation}
        \mathcal{L}_{\text{SupCon}} = \sum_{i=1}^N \frac{-1}{|\mathcal{P}(i)|} \sum_{p \in \mathcal{P}(i)} \log \frac{\exp\left(\text{sim}(\mathbf{e}_i, \mathbf{e}_p)/\tau\right)}{\sum_{a \in \mathcal{A}(i)} \exp\left(\text{sim}(\mathbf{e}_i, \mathbf{e}_a)/\tau\right)}
        \label{eq:supcon}
    \end{equation}
    where \( \mathcal{P}(i) = \{ j \ne i \mid y_j = y_i \} \) denotes the set of positive indices for anchor \( i \), and \( \mathcal{A}(i) = \{ j \ne i \} \) is the set of all other indices in the batch. The similarity function \( \text{sim}(\cdot, \cdot) \) is implemented as cosine similarity, and \( \tau > 0 \) is a temperature hyperparameter. We use \( \tau = 0.2 \) in our experiments.
    
    \item \textbf{Triplet Loss ($\mathcal{L}_{\text{Triplet}}$):}
    Triplet loss encourages an embedding space where examples of the same class are pulled closer together while pushing apart examples of different classes. Specifically, for an anchor embedding \( \mathbf{e}_a \), a positive sample \( \mathbf{e}_p \) (same class), and a negative sample \( \mathbf{e}_n \) (different class), the loss enforces a margin \( m > 0 \) between intra-class and inter-class distances:
    \begin{equation}
        \mathcal{L}_{\text{Triplet}} = \sum_{(\mathbf{e}_a, \mathbf{e}_p, \mathbf{e}_n)} \left[ \|\mathbf{e}_a - \mathbf{e}_p\|_2^2 - \|\mathbf{e}_a - \mathbf{e}_n\|_2^2 + m \right]_+,
        \label{eq:triplet}
    \end{equation}
    where \( [\cdot]_+ = \max(0, \cdot) \) denotes the hinge function. To improve convergence and avoid training on trivial triplets, we apply \textit{adaptive triplet mining} using a Multi-Similarity Miner~\citep{multisimilarityminer}. This strategy dynamically selects informative (hard or semi-hard) triplets from the mini-batch, based on the relative similarity of samples. In our experiments, we use a margin of \( m = 0.5 \).
\end{enumerate}
This hybrid loss structure capitalizes on the complementary strengths of each component: FVL for robust classification and centroid-based regularization, SupCon for global structure in the embedding space by contrasting multiple positives and negatives, and Triplet loss for fine-grained separation using specific anchor-positive-negative relationships.

\subsubsection{Projection Head for Robust Embeddings}
Following standard practice in CL frameworks~\citep{chen_simple_2020}, we implement a projection head \( p_\eta: \mathbb{R}^{d'} \rightarrow \mathbb{R}^{d_e} \) that maps intermediate representations to a space optimized for metric learning. Specifically, we transform the output of the second fully connected layer (\texttt{fc2}), where \( d' = 32 \), to an embedding dimension \( d_e = 8 \) used for supervised contrastive and variance-based objectives. The projection head is implemented as a two-layer MLP:
\begin{equation}
    \mathbb{R}^{32} \xrightarrow{\text{Linear}(16)} \text{BatchNorm} \xrightarrow{\text{ReLU}} \text{Linear}(8) \longrightarrow \mathbb{R}^{8}
\end{equation}
This component is trained jointly with the encoder \( h_\phi \), but its output is used \emph{only} for contrastive and regularization losses. The final classifier \( g_\psi \) (i.e., \texttt{fc3}) operates directly on the 32-dimensional features from \( h_\phi \), without the projection head, thus decoupling classification from CL.

\subsubsection{Overall Training Procedure}
Algorithm \ref{alg:QCLMixNet} presents a pseudocode that summarizes the complete end-to-end training procedure. During each training epoch, the model iterates over mini-batches, where we first apply our kNN-guided sample-aware dynamic mixup strategy to generate meaningful augmented samples. These augmented inputs are then passed through the core architecture, yielding two decoupled outputs: the final classification logits $\hat{y}$ for prediction, and low-dimensional embeddings $z$ from a dedicated projection head, used exclusively for metric learning. Importantly, the hybrid loss function is computed using the augmented representations but is supervised by the original, unmixed labels $y$. This design encourages the model to be robust against local perturbations in feature space while preserving the semantic identity of the anchor samples. Finally, model parameters are updated via backpropagation, and the best-performing checkpoint is preserved based on its macro-F1 score on a held-out validation set.

\subsection{Theoretical Analysis}
The design of our QCL-MixNet incorporates several components, each contributing to its overall learning capability and robustness.

\subsubsection{Expressiveness of QE Layers}
The QE layers incorporate sinusoidal projections and sigmoid-based gating (Eq.~\ref{eq:qe_proj}–\ref{eq:qe_ent}), yielding a non-linear transformation of the input space:
\begin{equation}
\mathbf{x}_{\text{ent}} = \sigma\left(\mathbf{x}^\top \sin(\boldsymbol{\theta})\right) \cdot \left( \mathbf{x} \odot \cos(\boldsymbol{\theta}) \right)
\end{equation}
where \( \boldsymbol{\theta} \in \mathbb{R}^d \) is a learnable parameter vector. This structure introduces both global feature interactions (via the dot product) and localized modulations (via the cosine-weighted projection).

\begin{proposition}[Expressiveness of QE Layer Composition]
Let \( f_{\text{QE}}(\cdot; \boldsymbol{\theta}) \) be the transformation induced by a QE layer. Then, the function class
\begin{equation}
\mathcal{F}_{\text{QE}} = \left\{ f(\mathbf{x}) = g \circ f_{\text{QE}}(\mathbf{x}) \mid g \in \mathcal{F}_{\text{MLP}} \right\}
\end{equation}
is a universal approximator over compact subsets of \( \mathbb{R}^d \), provided \( g \) is a standard feed-forward neural network with non-linear activation.
\end{proposition}
\begin{proof}[Sketch]
The QE transformation is continuous and differentiable in \( \mathbf{x} \) and \( \boldsymbol{\theta} \). Since it preserves the input dimensionality and introduces non-linearity, it acts as a learnable basis transformation. When composed with a fully connected MLP (which is a universal approximator), the overall function class remains dense in the space of continuous functions on compact domains (by the closure properties of universal approximators).
\end{proof}
Thus, QE layers contribute to the expressiveness of the network in a structured way, introducing spatially adaptive gates and sinusoidal modulation, which can improve data-fitting capacity while maintaining compact parameterization.

\subsubsection{Benefits of kNN-Guided Mixup}
Standard mixup~\citep{Zhangmixup2018} improves generalization by encouraging linear behavior between random pairs of training samples. However, such random interpolations may traverse regions far from the true data manifold, especially in class-imbalanced or multi-modal distributions. Our \textit{kNN-guided mixup} (Eq.~\ref{eq:knn_mixup_x}) refines this by restricting interpolation partners to semantically similar neighbors, thereby keeping synthetic samples closer to high-density regions of the data space.

\begin{proposition}[Manifold-Aware Regularization]
Let \( \mathcal{M} \subset \mathbb{R}^d \) be the data manifold. Compared to uniform mixup, kNN-guided mixup is more likely to generate samples \( \tilde{\mathbf{x}} = \lambda \mathbf{x}_i + (1 - \lambda) \mathbf{x}_j \) such that \( \tilde{\mathbf{x}} \in \mathcal{M} + \epsilon \), for small \( \epsilon > 0 \). This proximity to \( \mathcal{M} \) leads to stronger regularization and potentially tighter generalization bounds.
\end{proposition}
\begin{intuition} 
By sampling neighbors \( \mathbf{x}_j \in \mathcal{N}_k(\mathbf{x}_i) \) based on Euclidean proximity in the learned feature space, the interpolation respects local structure. In contrast, random mixup may interpolate between semantically disjoint classes, generating unrealistic data that can harm decision boundaries.
\end{intuition}

\subsubsection{Optimization Landscape with Hybrid Loss}
The hybrid loss function (Eq.~\ref{eq:hybrid_loss}) combines focal reweighting with metric-based representation learning and variance regularization. Each component shapes the loss landscape differently: (i) The focal term amplifies the gradient contribution of hard examples, counteracting class imbalance and encouraging escape from flat or poor local minima. (ii) The contrastive and triplet terms act on the embedding space, enforcing angular and margin-based class separation. (iii) The variance regularization terms induce intra-class compactness and inter-class repulsion, stabilizing feature distributions.

\begin{proposition}[Landscape Smoothing via Hybrid Loss]
Let \( \mathcal{L}_{\text{Hybrid}} = \alpha \mathcal{L}_{\text{Focal+Var}} + (1 - \alpha)(\mathcal{L}_{\text{SupCon}} + \mathcal{L}_{\text{Triplet}}) \). Then, under mild smoothness assumptions on the encoder and projection head, \( \mathcal{L}_{\text{Hybrid}} \) is locally Lipschitz and its gradient field encourages inter-class separation while preserving intra-class smoothness.
\end{proposition}

\begin{proof}[Sketch]
(i) \( \mathcal{L}_{\text{Focal}} \) is smooth away from \( p_t = 1 \); its gradients are steep near misclassified samples.
(ii) \( \mathcal{L}_{\text{SupCon}} \) is differentiable and has bounded gradients due to the softmax denominator.
(iii) \( \mathcal{L}_{\text{Triplet}} \) is piecewise linear with subgradients due to the hinge.
(iv) Variance losses are quadratic in embeddings, hence smooth and convex.
Therefore, the overall hybrid objective is piecewise smooth and exhibits gradient alignment toward class-separating embeddings, regularized by compactness constraints.
\end{proof}

\subsubsection{Role of Decoupled Representations}
To reconcile the objectives of classification and contrastive representation learning, we use a projection head \( p_\eta \) to decouple the encoder \( h'_\phi \)'s output from the embedding space used by the contrastive and regularization losses. Specifically, the encoder produces \( h'_\phi(\mathbf{x}) \in \mathbb{R}^{d'} \), which is used by the classifier \( g_\psi \), while the projection head maps to \( p_\eta(h'_\phi(\mathbf{x})) \in \mathbb{R}^{d_e} \) for contrastive loss computation.

\begin{proposition}[Representation Preservation and Task Decoupling]
Let \( \mathcal{L}_{\text{total}} = \mathcal{L}_{\text{classification}} + \mathcal{L}_{\text{contrastive}} \). Applying the contrastive loss directly on \( h'_\phi(\mathbf{x}) \) can suppress dimensions useful for classification. The use of a non-linear projection \( p_\eta \) enables \( h'_\phi \) to learn richer, task-relevant features while allowing contrastive alignment to occur in a separate, dedicated space.
\end{proposition}

\begin{proof}[Sketch]
As shown empirically by~\citep{chen_simple_2020}, projecting features into a lower-dimensional space before applying contrastive objectives improves downstream classification. This is because the encoder is freed from learning features solely shaped by contrastive geometry. The projection head learns a task-specific transformation optimized for metric learning, while the encoder focuses on preserving discriminative information for the main task.
\end{proof}

\begin{algorithm}[H]
\caption{QCL-MixNet for imbalanced tabular data}
\label{alg:QCLMixNet}
\begin{algorithmic}[1]
\footnotesize
\Function{QuantumEntanglement}{$x \in \mathbb{R}^{B \times D}$}
    \State Initialize $\theta \in \mathbb{R}^D$
    \State $x_{\text{proj}} \gets x \odot \cos(\theta)$ \Comment{Element-wise projection}
    \State $s \gets \sigma(x_{\text{proj}} \cdot \sin(\theta)^\top)$ \Comment{Scalar signal via matrix-vector product}
    \State $x_{\text{ent}} \gets x_{\text{proj}} \odot s$ \Comment{Entanglement with broadcasting}
    \State \Return $x_{\text{ent}}$
\EndFunction

\Function{QI\_NN}{$x$, return\_embedding}
    \State $x \gets \text{QuantumEntanglement}_1(x)$
    \State $x' \gets \text{reshape}(x, (B, 1, D))$ \Comment{Unsqueeze for attention}
    \State $x_{\text{attn}}, \_ \gets \text{MultiheadAttention}(x', x', x')$
    \State $x \gets x + \text{reshape}(x_{\text{attn}}, (B, D))$ \Comment{Squeeze after attention}
    \State $x \gets \text{LeakyReLU}(\text{BN}_1(\text{FC}_1(x)))$
    \State $x \gets \text{QuantumEntanglement}_2(x)$
    \State $x \gets \text{LeakyReLU}(\text{BN}_2(\text{FC}_2(x)))$
    \If{return\_embedding}
        \State \Return ProjectionHead($x$)
    \Else
        \State \Return FC3($x$)
    \EndIf
\EndFunction

\Function{SampleAwareMixup}{$x, y, \alpha, k$}
    \State For each $x_i$: find its k-nearest neighbors and select one, $x_j$.
    \State For each $y_i$: select a label from a random permutation, $y_p$. \Comment{Note: $y_p$ is NOT $y_j$}
    \State $\lambda \sim \text{Beta}(\alpha, \alpha)$; $\lambda \gets \max(\lambda, 1 - \lambda)$
    \State $x_{\text{mix}} \gets \lambda x_i + (1 - \lambda)x_j$
    \State $y_{\text{mix}} \gets \lambda \cdot \text{one\_hot}(y_i) + (1 - \lambda) \cdot \text{one\_hot}(y_p)$
    \State \Return $x_{\text{mix}}, y_{\text{mix}}$
\EndFunction

\Function{FocalVarianceLoss}{$\hat{y}, y, z$}
    \State $L_{\text{CE}} \gets \text{CrossEntropy}(\hat{y}, y)$
    \State $pt \gets \exp(-L_{\text{CE}})$
    \State $L_{\text{focal}} \gets (1 - pt)^\gamma \cdot L_{\text{CE}}$
    \State $L_{\text{var}} \gets$ MSE between embedding $z$ and its class centroid
    \State $L_{\text{sep}} \gets$ Penalty based on distance between class centroids
    \State \Return $\text{mean}(L_{\text{focal}} + \beta \cdot L_{\text{var}}) + L_{\text{sep}}$
\EndFunction

\Function{HybridLoss}{$\hat{y}, y, z$}
    \State $L_{\text{cls}} \gets \text{FocalVarianceLoss}(\hat{y}, y, z)$
    \State $L_{\text{con}} \gets \text{SupConLoss}(z, y)$
    \State $L_{\text{triplet}} \gets \text{TripletLoss}(z, y)$
    \State \Return $\alpha L_{\text{cls}} + (1 - \alpha)(L_{\text{con}} + L_{\text{triplet}})$
\EndFunction

\Procedure{TrainModel}{}
    \For{epoch $=1$ to $100$}
        \For{each batch $(x, y)$ in training data}
            \State $x_{\text{mix}}, y_{\text{mix}} \gets \text{SampleAwareMixup}(x, y)$
            \State $\hat{y} \gets \text{QI\_NN}(x_{\text{mix}})$
            \State $z \gets \text{QI\_NN}(x_{\text{mix}}, \text{return\_embedding=True})$
            \State loss $\gets \text{HybridLoss}(\hat{y}, \mathbf{y}, z)$
 \Comment{Loss uses original `y`, not $y_{mix}$}
            \State Backpropagate and update model parameters
        \EndFor
\State Evaluate F1 on validation data
        \If{F1 improves}
            \State Save model as best checkpoint
        \EndIf
    \EndFor
\EndProcedure

\end{algorithmic}
\end{algorithm}

\begin{table*}[!ht]
\footnotesize
\centering
\begin{sc}
% \resizebox{\textwidth}{!}{%
\begin{threeparttable}
\caption{Details of 7 binary and 11 multi-class imbalanced tabular datasets.} \label{tab:1}

\begin{tabular}{@{}clcccc@{}}
\toprule
\textbf{Classification Type} &
  \textbf{Datasets} &
  \textbf{\# Instances} &
  \textbf{\# Features} &
  \textbf{\# Class} &
  \textbf{Imbalance Ratio} \\ \midrule
 &
  ecoli\tnote{1} &
  336 &
  7 &
  2 &
  8.71 \\
 &
  optical\_digits\tnote{2} &
  5620 &
  64 &
  2 &
  9.13 \\
 &
  satimage\tnote{3} &
  6435 &
  36 &
  2 &
  9.30 \\
 &
  pen\_digits\tnote{4} &
  10992 &
  16 &
  2 &
  9.42 \\
 &
  abalone\tnote{5} &
  4177 &
  10 &
  2 &
  9.72 \\
 &
  isolet\tnote{6} &
  7797 &
  617 &
  2 &
  12.00 \\
\multirow{-7}{*}{Binary} &
  arrhythmia\tnote{7} &
  452 &
  278 &
  2 &
  17.20 \\ \midrule
 &
  minerals\tnote{8} &
  3112 &
  140 &
  7 &
  95.20 \\
 &
  vehicle\tnote{9} &
  846 &
  19 &
  4 &
  1.10 \\
 &
  satimage\tnote{10} &
  6430 &
  37 &
  6 &
  2.45 \\
 &
  har\tnote{11} &
  10299 &
  562 &
  6 &
  1.38 \\
 &
  wine-quality-red\tnote{12} &
  1599 &
  12 &
  6 &
  68.00 \\
 &
  lymph\tnote{13} &
  148 &
  19 &
  4 &
  17.00 \\
 &
  one-hundred-plants-texture\tnote{14} &
  1599 &
  65 &
  100 &
  1.33 \\
 &
  balance-scale\tnote{15} &
  625 &
  5 &
  3 &
  5.80 \\
 &
  wine-quality-white\tnote{16} &
  4898 &
  12 &
  7 &
  440.00 \\
 &
  letter\tnote{17} &
  20000 &
  17 &
  26 &
  1.11 \\
\multirow{-11}{*}{Multi-class} &
  glass\tnote{18} &
  214 &
  10 &
  6 &
  7.50 \\ \bottomrule
\end{tabular}%
\begin{tablenotes}
\footnotesize
\item[1] \url{https://archive.ics.uci.edu/ml/datasets/Ecoli}
\item[2] \url{https://archive.ics.uci.edu/dataset/80/optical+recognition+of+handwritten+digits}
\item[3] \url{https://archive.ics.uci.edu/dataset/146/statlog+landsat+satellite}
\item[4] \url{https://archive.ics.uci.edu/ml/datasets/pen-based+recognition+of+handwritten+digits}
\item[5] \url{https://archive.ics.uci.edu/ml/datasets/abalone}
\item[6] \url{https://archive.ics.uci.edu/ml/datasets/isolet}
\item[7] \url{https://archive.ics.uci.edu/ml/datasets/arrhythmia}
\item[8] \url{https://www.kaggle.com/datasets/vinven7/comprehensive-database-of-minerals/data}
\item[9] \url{https://www.openml.org/d/54}
\item[10] \url{https://www.openml.org/d/182}
\item[11] \url{https://www.openml.org/d/1478}
\item[12] \url{https://www.openml.org/d/40691}
\item[13] \url{https://www.openml.org/d/10}
\item[14] \url{https://www.openml.org/d/1493}
\item[15] \url{https://www.openml.org/d/11}
\item[16] \url{https://www.openml.org/d/40498}
\item[17] \url{https://www.openml.org/d/6}
\item[18] \url{https://www.openml.org/d/41}
\end{tablenotes}
\end{threeparttable}
% }
\end{sc}
\end{table*}

\section{Experiments}
\label{sec:experiments}
\subsection{Datasets}
To thoroughly evaluate the effectiveness of our proposed framework, we benchmarked it on 18 publicly available imbalanced datasets from the \textit{UCI Machine Learning Repository} \citep{asuncion2007uci}, \textit{OpenML}~\citep{vanRijn}, and \textit{Kaggle}, comprising both binary and multi-class classification tasks. Table~\ref{tab:1} summarizes the detailed characteristics of these datasets, including the number of instances, features, classes, class imbalance ratio, and their sources. For binary classification, we selected seven datasets: \textit{ecoli}, \textit{optical\_digits}, \textit{satimage}, \textit{pen\_digits}, \textit{abalone}, \textit{isolet} and \textit{arrhythmia}. These datasets span diverse domains and present varying degrees of class imbalance (ranging from 8.71 to 17.20), which enables a rigorous evaluation of model robustness under challenging imbalance conditions. For multi-class classification, we used 11 datasets, including \textit{minerals}, \textit{vehicle}, \textit{satimage}, \textit{har}, \textit{wine-quality-red}, \textit{lymph}, \textit{one-hundred-plants-texture}, \textit{balance-scale}, \textit{wine-quality-white}, \textit{letter} and \textit{glass}. These datasets cover a wide range of class counts (from 3 to 100) and imbalance ratios (up to 440). Such diversity allows for evaluating the generalization capabilities of the proposed model in complex, real-world settings.

\subsection{Implementation}
\subsubsection{Training details and reproducibility measures}
To ensure fair comparisons and reproducibility across all models, we used a consistent 80:20 stratified train-test split with a fixed random seed (42) and a batch size of 64. Preprocessing involved \texttt{StandardScaler}\footnote{\url{https://scikit-learn.org/stable/modules/generated/sklearn.preprocessing.StandardScaler.html}}for feature normalization (zero mean and unit variance) and \texttt{LabelEncoder}\footnote{\url{https://scikit-learn.org/stable/modules/generated/sklearn.preprocessing.LabelEncoder.html}} for encoding class labels. We implemented \texttt{GridSearchCV}\footnote{\url{https://scikit-learn.org/stable/modules/generated/sklearn.model_selection.GridSearchCV.html}} with 3-fold cross-validation for hyperparameter tuning in ML models. All deep and QI models were trained for 100 epochs using the AdamW optimizer with an initial learning rate of $1\times10^{-3}$ and weight decay of $1\times10^{-5}$. For QCL-based models, a \textit{1-cycle learning rate policy scheduler}\footnote{\url{https://docs.pytorch.org/docs/stable/generated/torch.optim.lr_scheduler.OneCycleLR.html}} was applied with a maximum learning rate of $1\times10^{-2}$ to stabilize training. GNN models dynamically constructed k-nearest neighbor graphs (k = 5) using torch-cluster during both training and evaluation, which enables GNNs to model local feature interactions in non-explicit graph structures typical of tabular data.. Macro F1 score was computed after each epoch on the held-out test set, and the best-performing model checkpoint was retained for final evaluation.

\subsubsection{Hardware Setup}
All experiments were conducted on a system equipped with an Intel(R) Xeon(R) CPU (4 vCPUs @ 2.0 GHz, 30 GB RAM) and dual NVIDIA T4 GPUs (16 GB VRAM, 2560 CUDA cores each), enabling efficient parallel training for deep and graph-based models. The implementation leveraged \textit{PyTorch 2.2.0} \citep{paszke2019pytorch}  for DL, \textit{scikit-learn 1.4.2} \citep{pedregosa2011scikit}  for classical ML models, and \textit{PyTorch Geometric 2.6.1} \citep{fey_fast_2019} in combination with \textit{Torch Cluster 1.6.3} for implementing graph-based neural architectures. All code was developed and executed using \textit{Python 3.10.12} in a Linux-based environment. Visualizations were created using \textit{Matplotlib 3.8.4} \citep{4160265} and \textit{Seaborn 0.13.2} \citep{waskom2021seaborn}.

\subsection{Benchmark Protocol}
\subsubsection{Evaluation metrics}
In order to evaluate the performance of our models on both binary and multi-class classification tasks, we implemented four standard evaluation metrics: Accuracy, Macro Average Precision (maP), Macro Average Recall (maR), and Macro Average  F1-score (maF1). These metrics are especially useful in uneven class distribution scenarios, as the macro-averaged scores assign equal weight to each class irrespective of frequency and thus provide a balanced evaluation. 

\paragraph{Accuracy} is the proportion of correctly classified instances among the total number of samples.
\begin{equation}
\text{Accuracy} = \frac{TP + TN}{TP + TN + FP + FN}
\end{equation}
\paragraph{maP} computes precision independently for each class and then takes the unweighted mean.
\begin{equation}
\text{maP} = \frac{1}{C} \sum_{i=1}^{C} \frac{TP_i}{TP_i + FP_i}
\end{equation}
\paragraph{maR} calculates recall for each class and averages them equally.
\begin{equation}
\text{maR} = \frac{1}{C} \sum_{i=1}^{C} \frac{TP_i}{TP_i + FN_i}
\end{equation}
\paragraph{maF1} is the harmonic mean of macro precision and macro recall. This metric combines precision and recall into a single value, treating both equally. By considering both metrics, it provides a more balanced evaluation of a \text{model's performance}.
\begin{equation}
\text{maF1} = \frac{2 \times \text{maP} \times \text{maR}}{\text{maP} + \text{maR}}
\end{equation}
Here, $C$ is the number of classes, and $TP_i$, $FP_i$, $FN_i$ represent the true positives, false positives, and false negatives for class $i$, respectively.

\subsection{Baseline Models}
We benchmarked our framework against 20 strong baseline models, including 8 ML, 7 DL, and 5 GNN models.
We implemented ML models including Extreme Gradient Boosting (XGBoost), Balanced Random Forest (Balanced RF), Support Vector Machine with SMOTE oversampling (SVM (SMOTE)), Decision Tree (DT), Random Forest (RF), Gradient Boosting (GB), Logistic Regression (LR), and kNN. Key hyperparameters (e.g., number of estimators, learning rate, regularization terms) were optimized via grid search.

The DL models consist of a MLP with two hidden layers (128 and 64 units) and batch normalization, a ResNet with three residual blocks, a Gated Recurrent Unit (GRU) and a Long Short-Term Memory (LSTM) network with 128 hidden units, their bidirectional variants (BiGRU and BiLSTM, 128 units per direction), and a Convolutional Neural Network (CNN) comprising two 1D convolutional layers (32 and 64 filters, kernel size 3).

To evaluate graph-structured representations, we adopted five GNNs implemented Graph Convolutional Network (GCN) \citep{GCN}, GraphSAGE \citep{graphSage}, Graph Attention Network (GAT) \citep{GAT}, Graph Isomorphism Network (GIN) \citep{GIN}, and Jumping Knowledge Network (JKNet) \citep{jknet}. The GNN architectures implemented in this study comprise two layers (Except for JKNet, which uses three layers) with 64 hidden units each, followed by a fully connected output layer for classification. All models were trained on identical data splits with consistent evaluation protocols for fair comparison.

%%%%%%%%%%% Benchmarking tables %%%%%%%%%%%%%%%


\let \clearpage

{\small\tabcolsep=5pt  % hold it local
\footnotesize
\begin{sc}
\begin{longtable}{@{}lllcccc@{}}
\caption{Performance benchmarking results for 7 binary tabular datasets with different class imbalance ratios. \textbf{Bold} indicates the best performance and {\ul{underline}} indicates the second best performance. `Diff' indicates the absolute improvement of QCL-MixNet over either the best-performing or the second-best-performing baseline model for each dataset, depending on which is more relevant in each scenario.}
\label{tab:2}\\
\toprule[1.5pt]
\textbf{Dataset}  & \textbf{Model Category}   & \textbf{Model}  & \textbf{Accuracy (\textcolor{mygreen}{$\uparrow$})} & \textbf{maP (\textcolor{mygreen}{$\uparrow$})} & \textbf{maR (\textcolor{mygreen}{$\uparrow$})} & \textbf{maF1 (\textcolor{mygreen}{$\uparrow$})} \\* \midrule[1pt]
\endfirsthead

\caption*{Table \ref{tab:2} (Continued): Performance benchmarking results for 7 binary tabular datasets with different class imbalance ratios.}\\
\toprule[1.5pt]
\textbf{Dataset} & \textbf{Model Category} & \textbf{Model}    & \textbf{Accuracy (\textcolor{mygreen}{$\uparrow$})} & \textbf{maP (\textcolor{mygreen}{$\uparrow$})} & \textbf{maR (\textcolor{mygreen}{$\uparrow$})} & \textbf{maF1 (\textcolor{mygreen}{$\uparrow$})} \\* \midrule[1pt]
\endhead

\hline
\endfoot
%
% \endlastfoot
%
\multirow{22}{*}{ecoli}       & \textbf{Proposed}              & \textbf{QCL-MixNet}  & \textbf{0.96}          & \ul{0.90}     & 0.85     & \ul{0.87}      \\ \cmidrule(l){2-7} 
                              & \multirow{8}{*}{ML}  & XGBoost     & 0.90          & 0.45     & 0.50     & 0.47      \\
                              &                       & Balanced RF & 0.82          & 0.67     & 0.84     & 0.70      \\
                              &                       & SVM (SMOTE) & 0.90          & 0.74     & \ul{0.88}     & 0.79      \\
                              &                       & DT          & \ul{0.94}          & 0.88     & 0.78     & 0.82      \\
                              &                       & RF          & 0.90          & 0.45     & 0.50     & 0.47      \\
                              &                       & GB          & 0.93          & 0.81     & 0.77     & 0.79      \\
                              &                       & LR          & 0.90          & 0.45     & 0.50     & 0.47      \\
\bottomrule
\end{longtable}
\end{sc}
}
\normalsize

\clearpage

{\small\tabcolsep=5pt  % hold it local
\footnotesize
\begin{sc}
\begin{longtable}{@{}cclcccc@{}}
\caption{Performance benchmarking results for 11 multi-class tabular datasets with different class imbalance ratios. \textbf{Bold} indicates the best performance and {\ul{underline}} indicates the second best performance. `Diff' indicates the absolute improvement of QCL-MixNet over either the best-performing or the second-best-performing baseline model for each dataset, depending on which is more relevant in each scenario.}
\label{tab:3}\\
\toprule[1.5pt]
\textbf{Dataset}  & \textbf{Model Category}   & \textbf{Model}  & \textbf{Accuracy (\textcolor{mygreen}{$\uparrow$})} & \textbf{maP (\textcolor{mygreen}{$\uparrow$})} & \textbf{maR (\textcolor{mygreen}{$\uparrow$})} & \textbf{maF1 (\textcolor{mygreen}{$\uparrow$})} \\* \midrule[1pt]
\endfirsthead

\caption*{Table \ref{tab:3} (Continued): Performance benchmarking results for 11 multi-class tabular datasets with different class imbalance ratios.}\\
\toprule[1.5pt]
\textbf{Dataset} & \textbf{Model Category} & \textbf{Model}    & \textbf{Accuracy (\textcolor{mygreen}{$\uparrow$})} & \textbf{maP (\textcolor{mygreen}{$\uparrow$})} & \textbf{maR (\textcolor{mygreen}{$\uparrow$})} & \textbf{maF1 (\textcolor{mygreen}{$\uparrow$})} \\* \midrule[1pt]
\endhead

\hline
\endfoot
%
% \endlastfoot

\multirow{22}{*}{Minerals} &
  \textbf{Proposed} &
  \cellcolor{green!20}\textbf{QCL-MixNet} &
  \cellcolor{lightblue}\textbf{0.87} &
  \cellcolor{lightblue}\textbf{0.84} &
  \cellcolor{lightblue}\textbf{0.87} &
  \cellcolor{lightblue}\textbf{0.85} \\ \cmidrule(l){2-7}
                           & \multirow{8}{*}{ML}  & XGBoost     & \ul{\textbf{0.86}} & \textbf{0.79} & \ul{\textbf{0.86}} & \textbf{0.82} \\
                           &                       & Balanced RF & 0.75                & 0.80          & 0.75                & 0.77          \\
                           &                       & SVM (SMOTE) & 0.82                & \ul{0.83}    & 0.82                & 0.82          \\
                           &                       & DT          & 0.77                & 0.61          & 0.77                & 0.67          \\
                           &                       & RF          & 0.84                & 0.80          & 0.84                & 0.82          \\
                           &                       & GB          & 0.84                & 0.82          & 0.84                & \ul{0.83}    \\
                           &                       & LR          & 0.83                & 0.77          & 0.83                & 0.79          \\
\bottomrule
\end{longtable}
\end{sc}
}
\normalsize

% \usepackage{booktabs}
% \usepackage{multirow}
% \usepackage{graphicx}
% \usepackage[normalem]{ulem}

\begin{table}[!ht]
{\small\tabcolsep=5pt  % hold it local
\footnotesize
\begin{sc}
\caption{Ablation study results for 7 binary tabular datasets. \textbf{Bold} indicates the best performance and {\ul{underline}} indicates the second best performance. `Diff' indicates the absolute improvement of QCL-MixNet over either the best-performing or the second-best-performing baseline model for each dataset, depending on which is more relevant in each scenario.}
\label{tab:ablation-binary}
% \resizebox{\textwidth}{!}{%
\begin{tabular}{@{}cccccc@{}}
\toprule[1.5pt]
\textbf{Dataset}  &  \textbf{Model}  & \textbf{Accuracy (\textcolor{mygreen}{$\uparrow$})} & \textbf{maP (\textcolor{mygreen}{$\uparrow$})} & \textbf{maR (\textcolor{mygreen}{$\uparrow$})} & \textbf{maF1 (\textcolor{mygreen}{$\uparrow$})} \\ \midrule[1pt]
\multirow{5}{*}{ecoli}           & \textbf{QCL-MixNet (Full)} & \textbf{0.96}          & \textbf{0.90}     & \textbf{0.85}     & \textbf{0.87}      \\
                                 & QCL-MixNet (No Quantum)    & 0.90                   & 0.45              & 0.50              & 0.47               \\
                                 & QCL-MixNet (No Mixup)      & \ul{0.94}             & \ul{0.84}        & \ul{0.84}        & \ul{0.84}         \\
                                 & QCL-MixNet (No Attention)  & 0.90                   & 0.45              & 0.50              & 0.47               \\  \cmidrule(l){2-6}
                                 & Diff                       & \textcolor{mygreen}{+0.02}                  & \textcolor{mygreen}{+0.06}             & \textcolor{mygreen}{+0.01}             & \textcolor{mygreen}{+0.03}              \\ \midrule
\multirow{5}{*}{optical\_digits} & \textbf{QCL-MixNet (Full)} & \textbf{0.99}          & 0.96              & \textbf{0.99}     & \textbf{0.98}      \\
                                 & QCL-MixNet (No Quantum)    & \textbf{0.99}          & \textbf{0.99}     & \ul{0.98}        & \textbf{0.98}      \\
                                 & QCL-MixNet (No Mixup)      & \ul{0.98}             & 0.93              & 0.95              & \ul{0.94}         \\
                                 & QCL-MixNet (No Attention)  & \ul{0.98}             & \ul{0.98}        & 0.90              & 0.93               \\  \cmidrule(l){2-6}
                                 & Diff                       & \textcolor{blue}{0.00}                   & \textcolor{myred}{-0.03}
                                 & \textcolor{mygreen}{+0.01}             & \textcolor{blue}{0.00}               \\ \midrule
\multirow{5}{*}{satimage}    & \textbf{QCL-MixNet (Full)} & \textbf{0.95} & \textbf{0.86} & \textbf{0.82} & \textbf{0.84} \\
                                 & QCL-MixNet (No Quantum)    & \ul{0.93}             & 0.83              & \ul{0.73}        & \ul{0.76}         \\
                                 & QCL-MixNet (No Mixup)      & \ul{0.93}             & \ul{0.84}        & 0.72              & \ul{0.76}         \\
                                 & QCL-MixNet (No Attention)  & 0.92                   & 0.83              & 0.65              & 0.70               \\  \cmidrule(l){2-6}
                                 & Diff                       & \textcolor{mygreen}{+0.02}                  & \textcolor{mygreen}{+0.02}             & \textcolor{mygreen}{+0.09}             & \textcolor{mygreen}{+0.08}              \\ \midrule
\multirow{5}{*}{pen\_digits} & \textbf{QCL-MixNet (Full)} & \textbf{1.00} & \textbf{1.00} & \textbf{1.00} & \textbf{1.00} \\
                                 & QCL-MixNet (No Quantum)    & \textbf{1.00}          & \textbf{1.00}     & \textbf{1.00}     & \textbf{1.00}      \\
                                 & QCL-MixNet (No Mixup)      & \textbf{1.00}          & \textbf{1.00}     & \textbf{1.00}     & \textbf{1.00}      \\
                                 & QCL-MixNet (No Attention)  & \textbf{1.00}          & \ul{0.99}        & \textbf{1.00}     & \ul{0.99}         \\  \cmidrule(l){2-6}
                                 & Diff                       & \textcolor{blue}{0.00}                   & \textcolor{mygreen}{+0.01}             & \textcolor{blue}{0.00}              & \textcolor{mygreen}{+0.01}              \\ \midrule
\multirow{5}{*}{abalone}         & \textbf{QCL-MixNet (Full)} & \textbf{0.91}          & \textbf{0.45}     & \textbf{0.50}     & \textbf{0.48}      \\
                                 & QCL-MixNet (No Quantum)    & \textbf{0.91}          & \textbf{0.45}     & \textbf{0.50}     & \textbf{0.48}      \\
                                 & QCL-MixNet (No Mixup)      & \textbf{0.91}          & \textbf{0.45}     & \textbf{0.50}     & \textbf{0.48}      \\
                                 & QCL-MixNet (No Attention)  & \textbf{0.91}          & \textbf{0.45}     & \textbf{0.50}     & \textbf{0.48}      \\  \cmidrule(l){2-6}
                                 & Diff                       & \textcolor{blue}{0.00}                   & \textcolor{blue}{0.00}              & \textcolor{blue}{0.00}              & \textcolor{blue}{0.00}               \\ \midrule
\multirow{5}{*}{isolet}          & \textbf{QCL-MixNet (Full)} & \textbf{0.99}          & \textbf{0.96}     & 0.94              & \textbf{0.95}      \\
                                 & QCL-MixNet (No Quantum)    & \textbf{0.99}          & \ul{0.95}        & \textbf{0.96}     & \textbf{0.95}      \\
                                 & QCL-MixNet (No Mixup)      & \ul{0.98}             & 0.93              & \ul{0.95}        & \ul{0.94}         \\
                                 & QCL-MixNet (No Attention)  & \textbf{0.99}          & 0.94              & \textbf{0.96}     & \textbf{0.95}      \\ \cmidrule(l){2-6}
                                 & Diff                       & \textcolor{blue}{0.00}                   & \textcolor{mygreen}{+0.01}             & \textcolor{myred}{-0.02}             & \textcolor{blue}{0.00}               \\ \midrule
\multirow{5}{*}{arrhythmia}      & \textbf{QCL-MixNet (Full)} & \ul{0.91}             & \textbf{0.62}     & \textbf{0.67}     & \textbf{0.64}      \\
                                 & QCL-MixNet (No Quantum)    & \textbf{0.95}          & \ul{0.47}        & \ul{0.50}        & \ul{0.49}         \\
                                 & QCL-MixNet (No Mixup)      & \textbf{0.95}          & \ul{0.47}        & \ul{0.50}        & \ul{0.49}         \\
                                 & QCL-MixNet (No Attention)  & \textbf{0.95}          & \ul{0.47}        & \ul{0.50}        & \ul{0.49}         \\ \cmidrule(l){2-6} 
                                 & Diff                       & \textcolor{myred}{-0.04}                  & \textcolor{mygreen}{+0.15}             & \textcolor{mygreen}{+0.17}             & \textcolor{mygreen}{+0.15}              \\ \bottomrule[1.5pt]
\end{tabular}%
% }
\end{sc}
}
\end{table}

\begin{figure*}[!ht]
\centering
\begin{subfigure}[b]{0.494\linewidth} \centering
    \includegraphics[width=\linewidth]{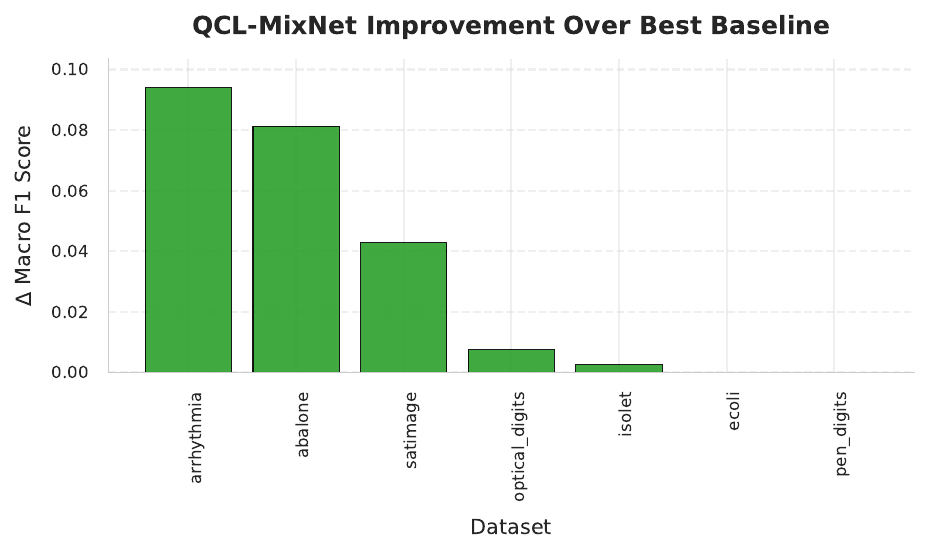}
    \caption{Improvement of QCL-MixNet over the best performing baseline on 7 binary datasets}
    \label{fig:improvement-binary}
\end{subfigure}
\hfill
\begin{subfigure}[b]{0.5\linewidth} \centering
    \includegraphics[width=0.95\linewidth]{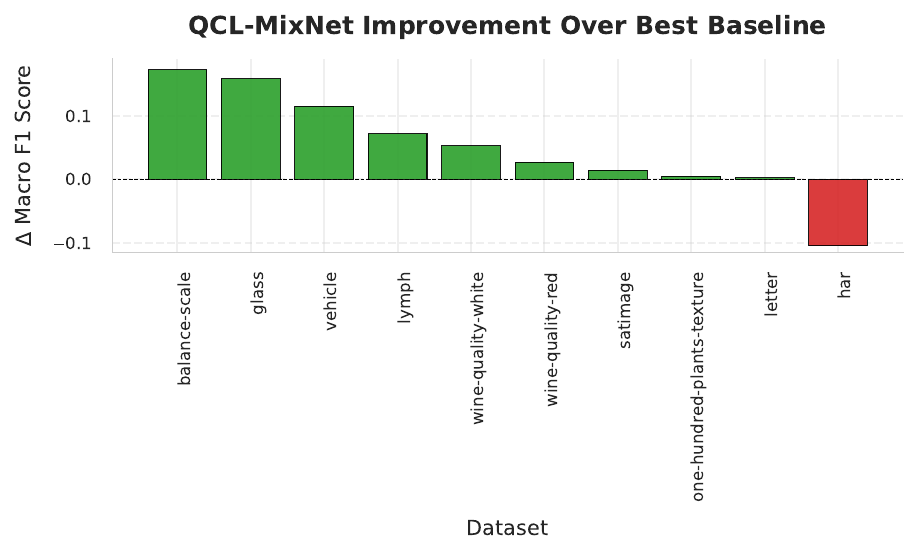}
   \caption{Improvement of QCL-MixNet over the best performing baseline on 11 multiclass datasets}
    \label{fig:improvement-multiclass}
\end{subfigure}
\caption{maF1-score improvement of QCL-MixNet over the best performing baseline on binary and multiclass datasets.}
\label{fig:improvement_overall}
\end{figure*}

\begin{figure*}[!ht]
\centering
\begin{subfigure}[b]{0.49\linewidth} \centering
    \includegraphics[width=\linewidth]{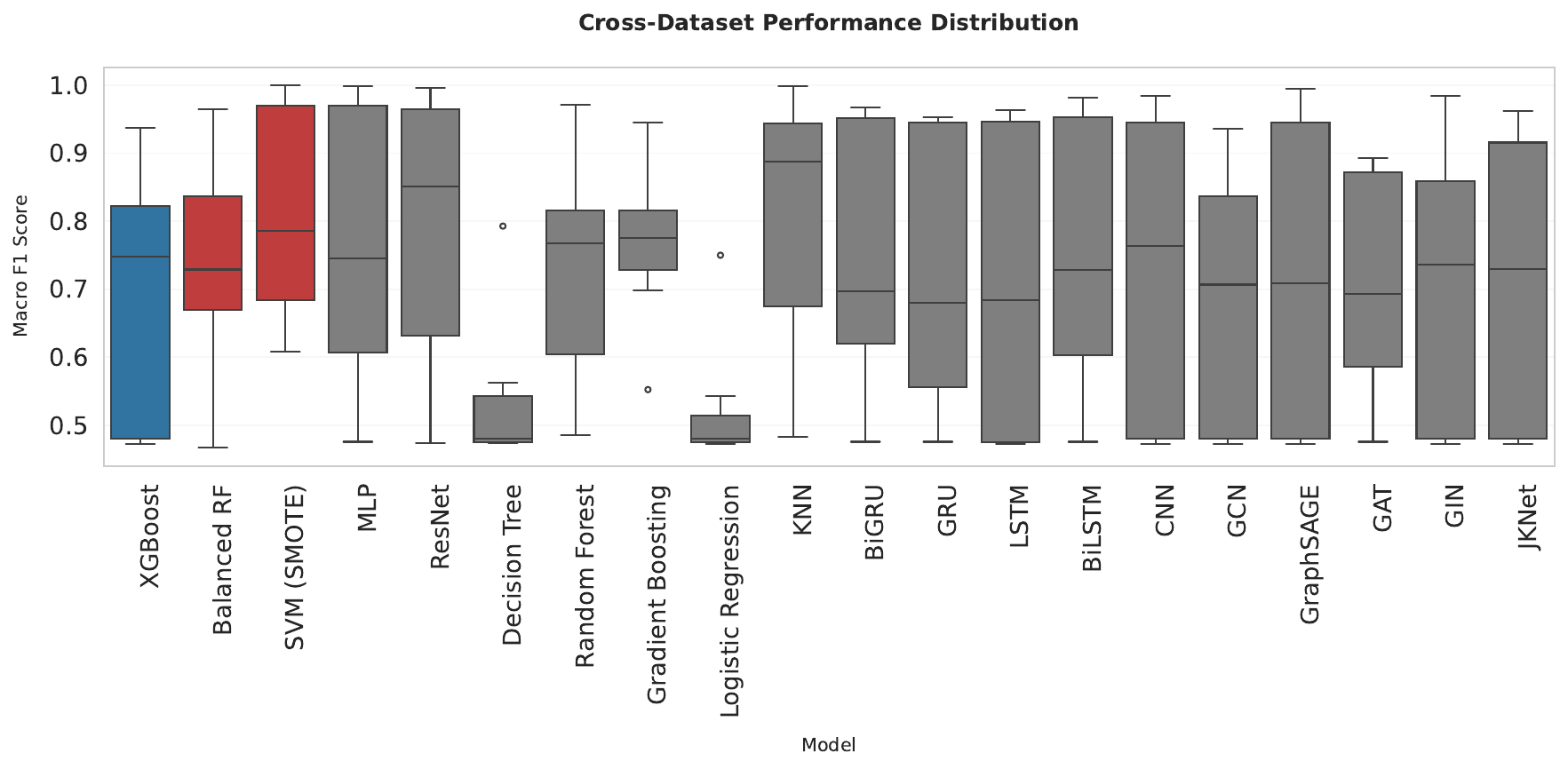}
    \caption{Performance distribution on binary datasets}
    \label{fig:cross-binary}
\end{subfigure}
\hfill
\begin{subfigure}[b]{0.49\linewidth} \centering
    \includegraphics[width=\linewidth]{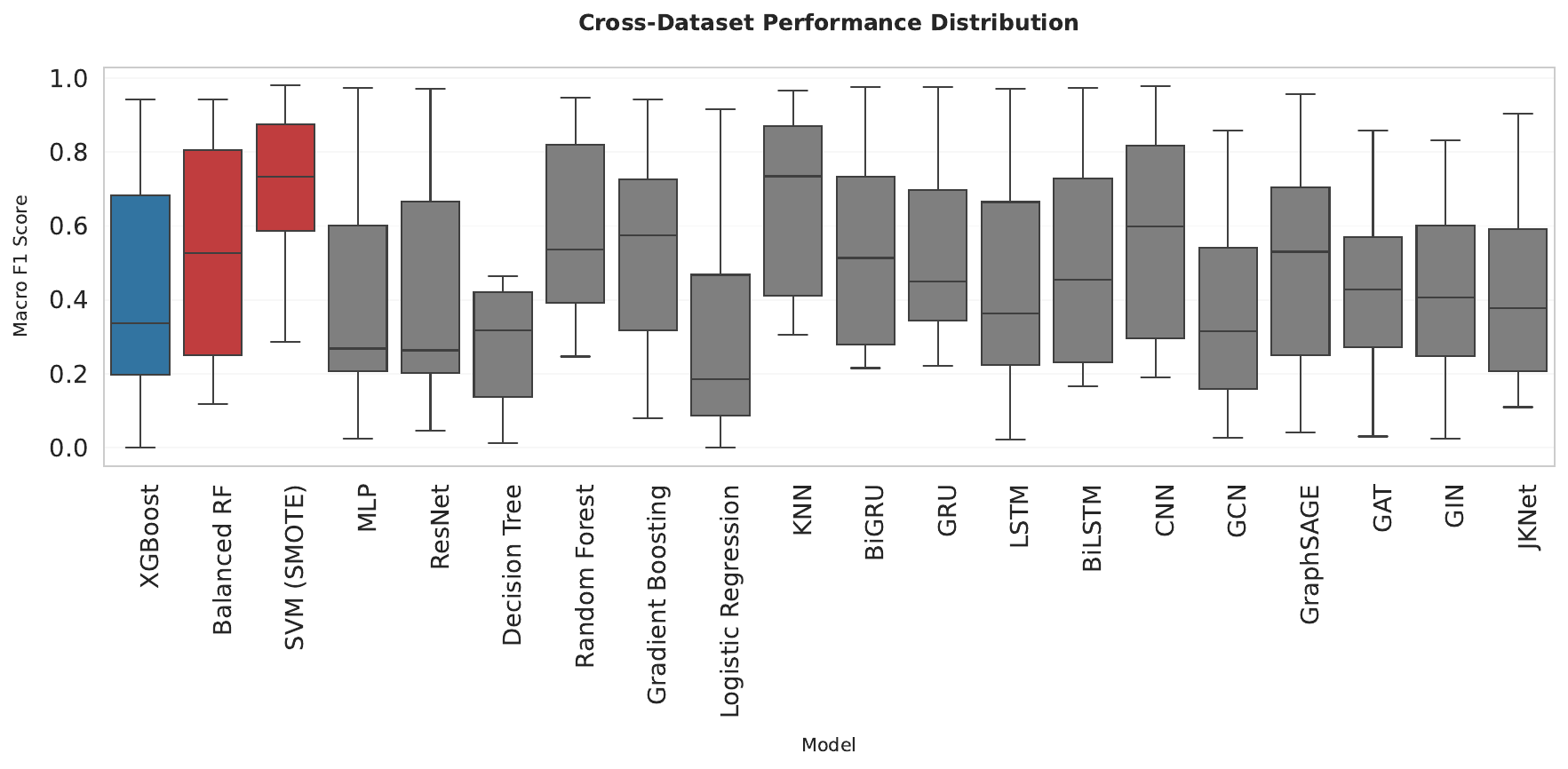}
   \caption{Performance distribution on multiclass datasets}
    \label{fig:cross-multi}
\end{subfigure}
\caption{Cross-dataset maF1 score distribution of various models on \textbf{(a)} 7 binary and \textbf{(b)} 11 multiclass datasets.}
\label{fig:cross_overall}
\end{figure*}

\begin{figure*}[!ht]
\centering
\begin{subfigure}[b]{0.49\textwidth}
    \includegraphics[width=\textwidth]{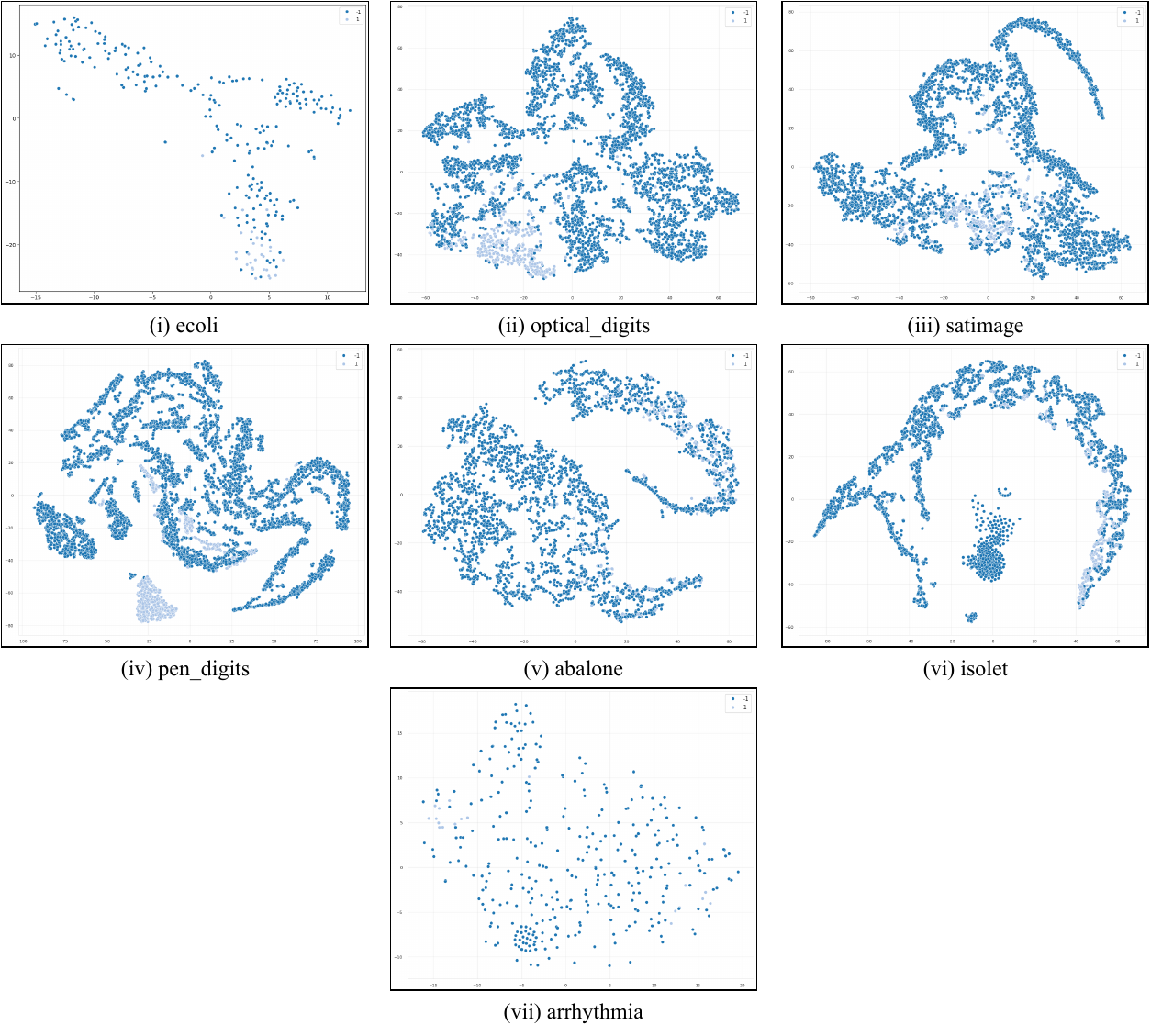}
    \caption{t-SNE embeddings on binary datasets}
    \label{tsne-binary}
\end{subfigure}
\hfill
\begin{subfigure}[b]{0.49\textwidth}
    \includegraphics[width=\textwidth]{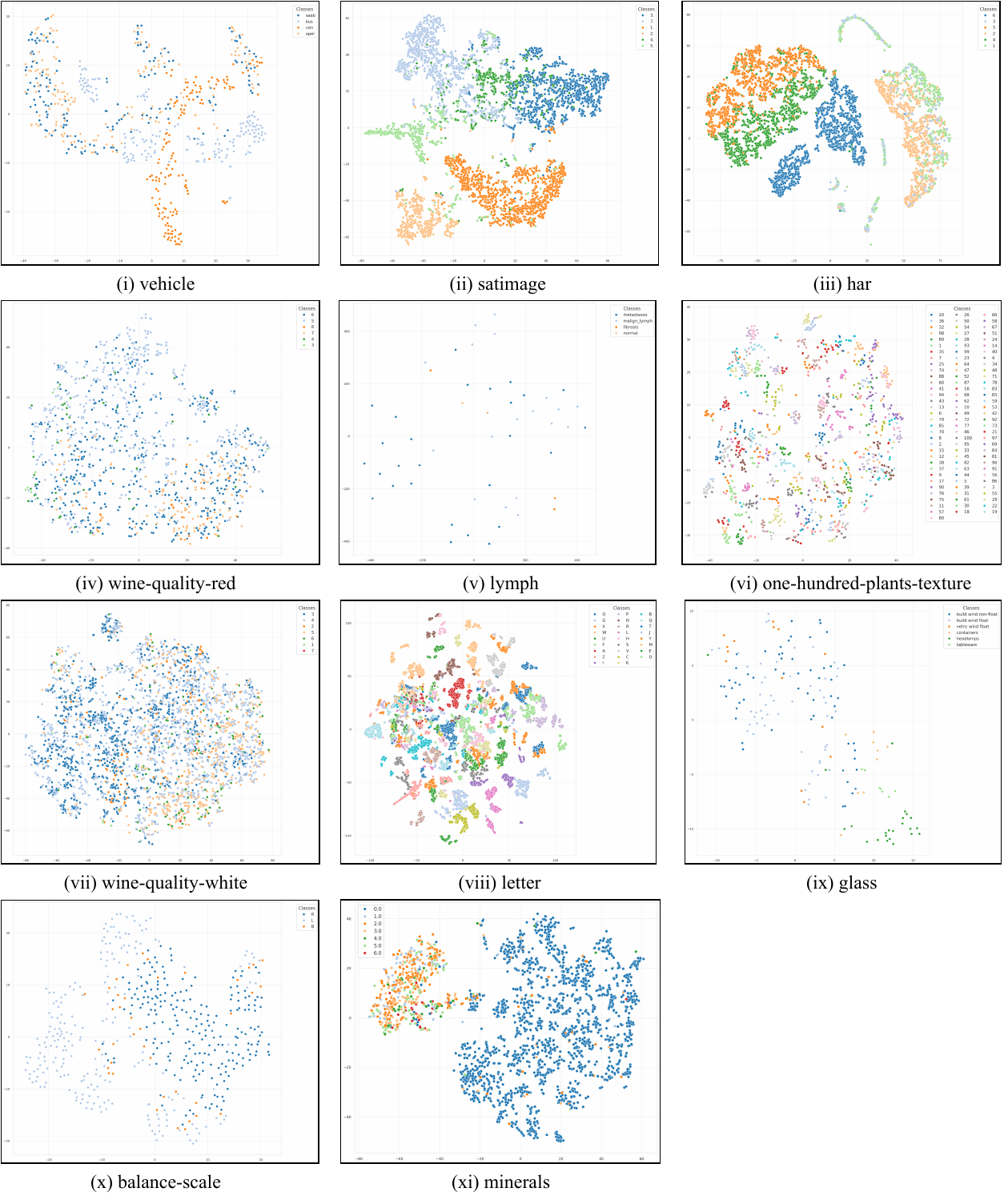}
    \caption{t-SNE embeddings on multiclass datasets}
    \label{tsne-multiclass}
\end{subfigure}
\caption{t-SNE plots showing prediction-space clustering of QCL-MixNet on binary and multiclass datasets. \textbf{(a)} presents t-SNE embeddings for binary classification tasks, while \textbf{(b)} shows t-SNE embeddings for multiclass classification tasks. Each plot reflects the distribution of model predictions in the latent space, highlighting the degree of class separation and cluster compactness achieved by QCL-MixNet.}
\label{fig:tsne_main}
\end{figure*}

%%%%%%%%%%%%%%%%%%%%%%%%%%%%%%%%%

%%%%%%%%%%% Ablation Tables %%%%%%%%

\begin{figure*}[!ht]
\centering
\begin{subfigure}[b]{0.32\linewidth} \centering
    \includegraphics[width=\linewidth]{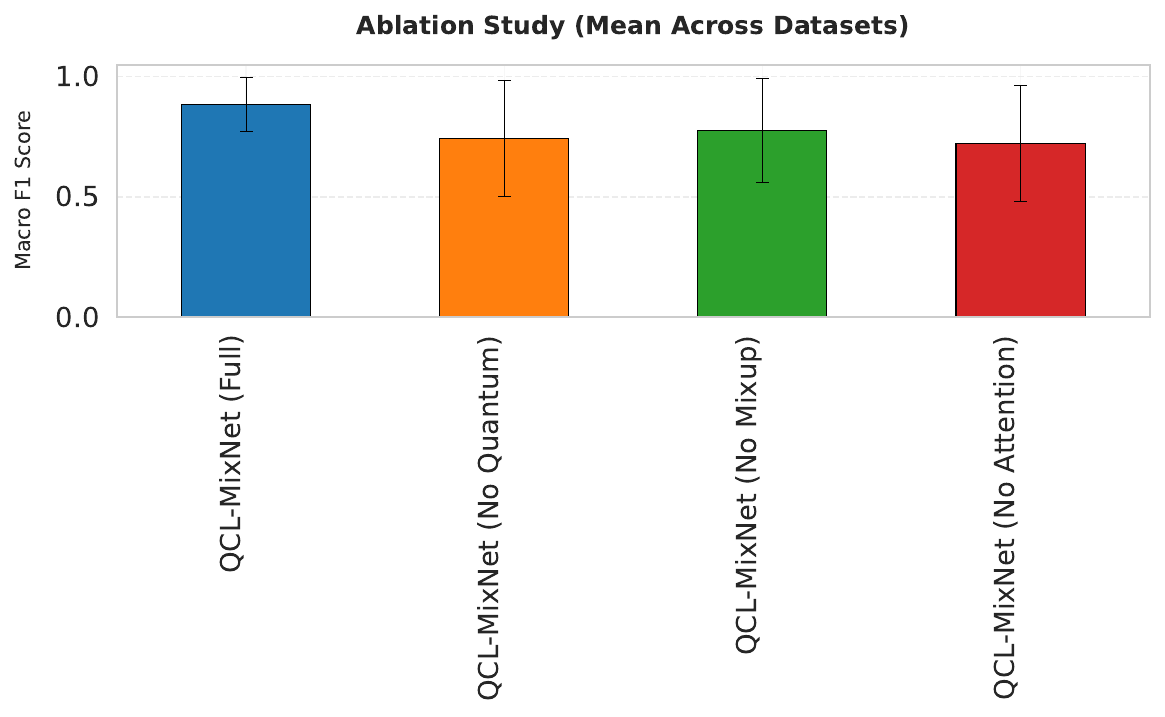}
    \caption{Mean across 7 binary datasets}
    \label{fig:ablation-binary}
\end{subfigure}
\hfill
\begin{subfigure}[b]{0.67\linewidth} \centering
    \includegraphics[width=0.48\linewidth]{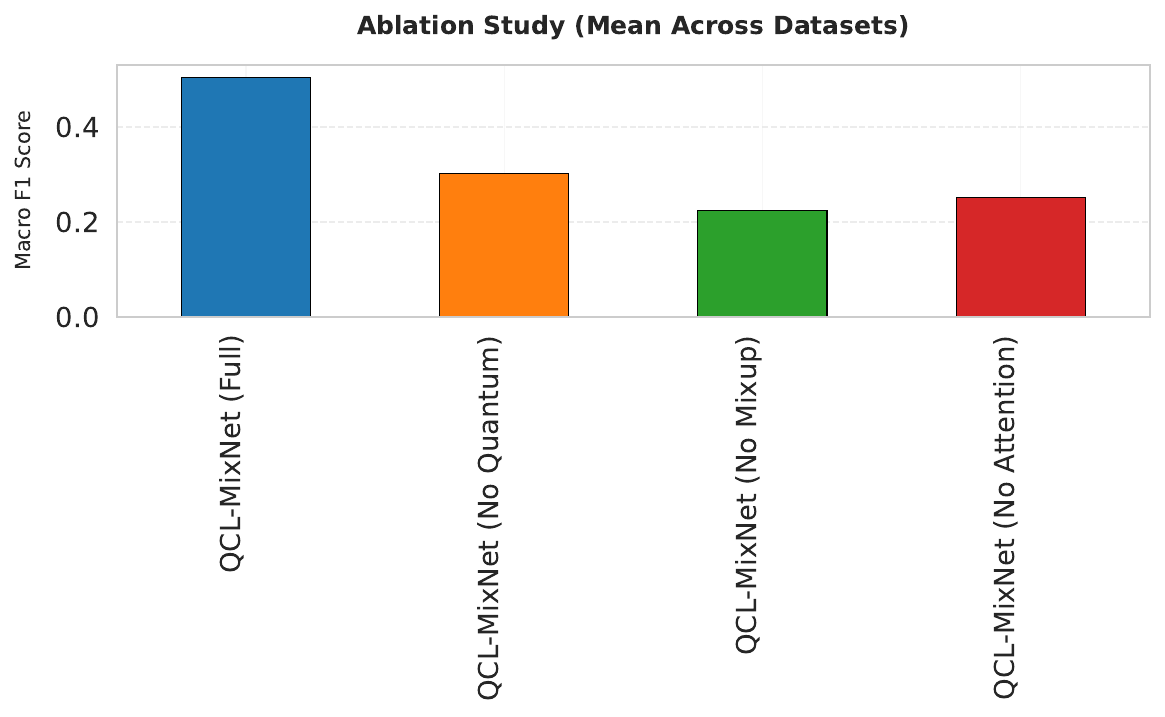}
   % \caption{Ablation study on mineral dataset}
   %  \label{fig:ablation-mineral}
% \end{subfigure}
% \hfill
% \begin{subfigure}{0.33\linewidth} \centering
    \includegraphics[width=0.48\linewidth]{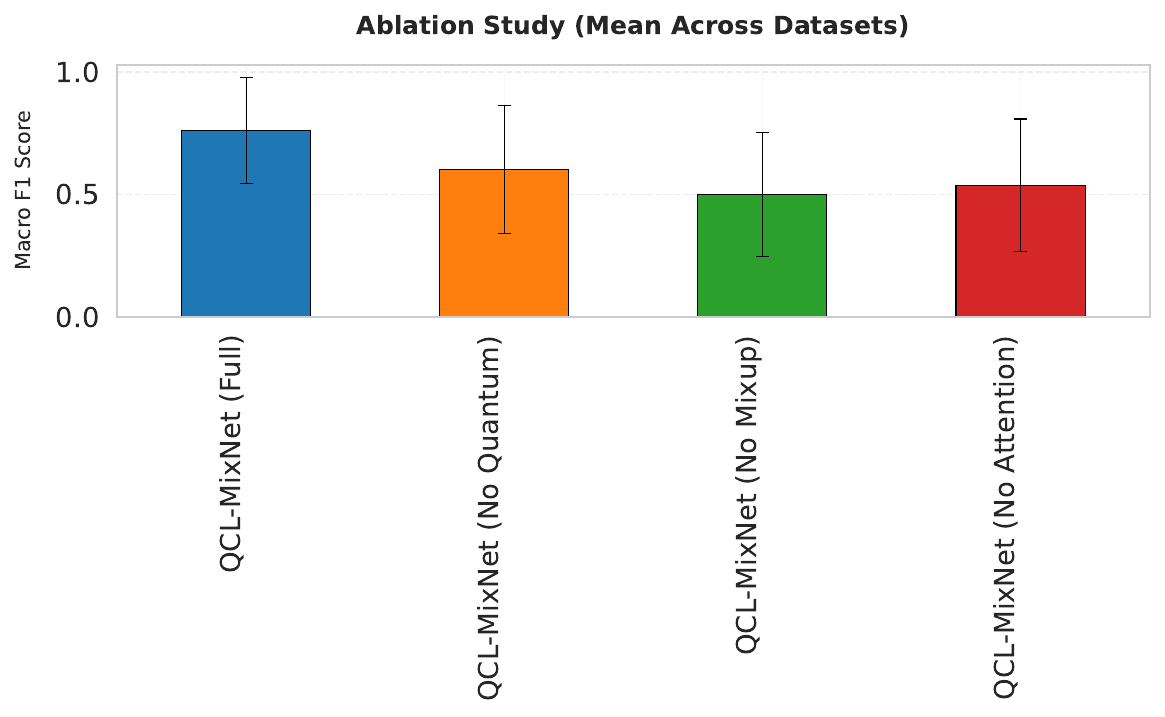}
   \caption{Ablation study on multiclass datasets (\textit{mineral} dataset on \textcolor{mygreen}{left} and the rest on \textcolor{myred}{right})}
    \label{fig:ablation-multi}
\end{subfigure}
\caption{Ablation study showing the impact of removing key components from QCL-MixNet. Performance is measured by maF1 score on \textbf{(a)} the mean of 7 binary datasets and \textbf{(b)} the single \textit{mineral} multiclass (\textcolor{mygreen}{left}) dataset and the mean of 10 other multiclass (\textcolor{myred}{right}) datasets. In all scenarios, the full model outperforms variants where the quantum-inspired embedding, mixup augmentation, or self-attention is removed. This confirms that each component contributes positively to the model's overall performance. \textit{Error bars} represent standard deviation across datasets.}
\label{fig:ablation_overall}
\end{figure*}

%%%%%%%%%%%%%%%%%%%%%%%%%%%%%%%%%

% \usepackage{booktabs}
% \usepackage{multirow}
% \usepackage{graphicx}
% \usepackage[normalem]{ulem}

{\small\tabcolsep=5pt  % hold it local
\footnotesize
\begin{sc}
\begin{longtable}{@{}cclcccc@{}}
\caption{Ablation study results for 11 multi-class tabular datasets. \textbf{Bold} indicates the best performance and {\ul{underline}} indicates the second best performance. `Diff' indicates the absolute improvement of QCL-MixNet over either the best-performing or the second-best-performing baseline model for each dataset, depending on which is more relevant in each scenario.}
\label{tab:ablation-multi}\\
\toprule[1.5pt]
\textbf{Dataset}  &  \textbf{Model}  & \textbf{Accuracy (\textcolor{mygreen}{$\uparrow$})} & \textbf{maP (\textcolor{mygreen}{$\uparrow$})} & \textbf{maR (\textcolor{mygreen}{$\uparrow$})} & \textbf{maF1 (\textcolor{mygreen}{$\uparrow$})} \\* \midrule[1pt]
\endfirsthead

\caption*{Table \ref{tab:ablation-multi} (Continued): Ablation study results for 11 multi-class tabular datasets.}\\
\toprule[1.5pt]
\textbf{Dataset} & \textbf{Model}    & \textbf{Accuracy (\textcolor{mygreen}{$\uparrow$})} & \textbf{maP (\textcolor{mygreen}{$\uparrow$})} & \textbf{maR (\textcolor{mygreen}{$\uparrow$})} & \textbf{maF1 (\textcolor{mygreen}{$\uparrow$})} \\* \midrule[1pt]
\endhead
\hline
\endfoot
%
% \endlastfoot

\multirow{5}{*}{Mineral}  & \textbf{QCL-MixNet (Full)} & \textbf{0.87} & \textbf{0.47} & \textbf{0.47} & \textbf{0.43} \\
                          & QCL-MixNet (No Quantum)    & \ul{0.84}    & \ul{0.33}    & \ul{0.31}    & \ul{0.3}     \\
                          & QCL-MixNet (No Mixup)      & 0.83          & 0.19          & 0.25          & 0.21          \\
                          & QCL-MixNet (No Attention)  & 0.83          & 0.23          & 0.27          & 0.24          \\ \cmidrule(l){2-6} 
                          & Diff                       & \textcolor{mygreen}{+0.03}         & \textcolor{mygreen}{+0.14}         & \textcolor{mygreen}{+0.16}         & \textcolor{mygreen}{+0.13}         \\ \midrule
\multirow{5}{*}{vehicle}  & \textbf{QCL-MixNet (Full)} & \textbf{0.85} & \textbf{0.86} & \textbf{0.85} & \textbf{0.85} \\
                          & QCL-MixNet (No Quantum)    & \ul{0.76}    & \ul{0.76}    & \ul{0.76}    & \ul{0.75}    \\
                          & QCL-MixNet (No Mixup)      & 0.62          & 0.61          & 0.63          & 0.6           \\
                          & QCL-MixNet (No Attention)  & 0.67          & 0.66          & 0.67          & 0.66          \\ \cmidrule(l){2-6} 
                          & Diff                       & \textcolor{mygreen}{+0.09}         & \textcolor{mygreen}{+0.1}          & \textcolor{mygreen}{+0.09}         & \textcolor{mygreen}{+0.1}          \\ \midrule
\multirow{5}{*}{satimage} & \textbf{QCL-MixNet (Full)} & \textbf{0.91} & \textbf{0.9}  & \textbf{0.89} & \textbf{0.91} \\
                          & QCL-MixNet (No Quantum)    & \ul{0.88}    & \ul{0.86}    & \ul{0.85}    & \ul{0.88}    \\
                          & QCL-MixNet (No Mixup)      & 0.87          & 0.84          & 0.84          & 0.87          \\
                          & QCL-MixNet (No Attention)  & 0.86          & 0.82          & 0.81          & 0.85          \\ \cmidrule(l){2-6} 
                          & Diff                       & \textcolor{mygreen}{+0.03}         & \textcolor{mygreen}{+0.04}         & \textcolor{mygreen}{+0.04}         & \textcolor{mygreen}{+0.03}         \\ \midrule
\multirow{5}{*}{har}      & \textbf{QCL-MixNet (Full)} & \ul{0.89}    & \ul{0.92}    & \ul{0.89}    & \ul{0.88}    \\
                          & QCL-MixNet (No Quantum)    & \textbf{0.98} & \textbf{0.98} & \textbf{0.98} & \textbf{0.98} \\
                          & QCL-MixNet (No Mixup)      & 0.64          & 0.56          & 0.66          & 0.55          \\
                          & QCL-MixNet (No Attention)  & 0.66          & 0.79          & 0.69          & 0.6           \\ \cmidrule(l){2-6}  \cmidrule(l){2-6} 
                          & Diff                       & \textcolor{myred}{-0.09}         & \textcolor{myred}{-0.06}         & \textcolor{myred}{-0.09}         & \textcolor{myred}{-0.1}          \\ \midrule
\multirow{5}{*}{wine-quality-red}           & \textbf{QCL-MixNet (Full)} & \textbf{0.65}          & \textbf{0.33}     & \textbf{0.32}     & \textbf{0.32}      \\
                          & QCL-MixNet (No Quantum)    & \ul{0.61}    & 0.3           & \ul{0.29}    & \ul{0.29}    \\
                          & QCL-MixNet (No Mixup)      & 0.57          & \ul{0.31}    & 0.26          & 0.26          \\
                          & QCL-MixNet (No Attention)  & 0.56          & 0.26          & 0.24          & 0.24          \\ \cmidrule(l){2-6} 
                          & Diff                       & \textcolor{mygreen}{+0.04}         & \textcolor{mygreen}{+0.02}         & \textcolor{mygreen}{+0.03}         & \textcolor{mygreen}{+0.03}         \\ \midrule
\multirow{5}{*}{lymph}    & \textbf{QCL-MixNet (Full)} & \textbf{0.73} & \textbf{0.84} & \textbf{0.79} & \ul{0.79}    \\
                          & QCL-MixNet (No Quantum)    & \ul{0.57}    & 0.2           & 0.33          & 0.25          \\
                          & QCL-MixNet (No Mixup)      & \ul{0.57}    & 0.36          & \ul{0.35}    & 0.32          \\
                          & QCL-MixNet (No Attention)  & \textbf{0.73} & \ul{0.82}    & \textbf{0.79} & \textbf{0.8}  \\ \cmidrule(l){2-6} 
                          & Diff                       & \textcolor{blue}{0.00}          & \textcolor{mygreen}{+0.02}         & \textcolor{blue}{0.00}          & \textcolor{myred}{-0.01}         \\ \midrule
\multirow{5}{*}{one-hundred-plants-texture} & \textbf{QCL-MixNet (Full)} & \textbf{0.83}          & \textbf{0.87}     & \textbf{0.84}     & \textbf{0.83}      \\
                          & QCL-MixNet (No Quantum)    & \ul{0.51}    & \ul{0.48}    & \ul{0.51}    & \ul{0.44}    \\
                          & QCL-MixNet (No Mixup)      & 0.2           & 0.14          & 0.19          & 0.14          \\
                          & QCL-MixNet (No Attention)  & 0.3           & 0.26          & 0.3           & 0.24          \\ \cmidrule(l){2-6} 
                          & Diff                       & \textcolor{mygreen}{+0.32}         & \textcolor{mygreen}{+0.39}         & \textcolor{mygreen}{+0.33}         & \textcolor{mygreen}{+0.39}         \\ \midrule
\multirow{5}{*}{balance-scale}              & \textbf{QCL-MixNet (Full)} & \textbf{0.94}          & \textbf{0.96}     & \textbf{0.76}     & \textbf{0.8}       \\
                          & QCL-MixNet (No Quantum)    & \ul{0.89}    & \ul{0.59}    & \ul{0.64}    & \ul{0.62}    \\
                          & QCL-MixNet (No Mixup)      & 0.88          & \ul{0.59}    & \ul{0.64}    & 0.61          \\
                          & QCL-MixNet (No Attention)  & 0.88          & \ul{0.59}    & \ul{0.64}    & 0.61          \\ \cmidrule(l){2-6} 
                          & Diff                       & \textcolor{mygreen}{+0.05}         & \textcolor{mygreen}{+0.37}         & \textcolor{mygreen}{+0.12}         & \textcolor{mygreen}{+0.18}         \\ \midrule
\multirow{5}{*}{wine-quality-white}         & \textbf{QCL-MixNet (Full)} & \textbf{0.58}          & \textbf{0.39}     & \textbf{0.34}     & \textbf{0.35}      \\
                          & QCL-MixNet (No Quantum)    & \ul{0.55}    & \ul{0.3}     & \ul{0.26}    & \ul{0.26}    \\
                          & QCL-MixNet (No Mixup)      & \ul{0.55}    & 0.24          & 0.23          & 0.23          \\
                          & QCL-MixNet (No Attention)  & \ul{0.55}    & 0.24          & 0.23          & 0.23          \\ \cmidrule(l){2-6} 
                          & Diff                       & \textcolor{mygreen}{+0.03}         & \textcolor{mygreen}{+0.09}         & \textcolor{mygreen}{+0.08}         & \textcolor{mygreen}{+0.09}         \\ \midrule
\multirow{5}{*}{letter}   & \textbf{QCL-MixNet (Full)} & \textbf{0.95} & \textbf{0.95} & \textbf{0.95} & \textbf{0.95} \\
                          & QCL-MixNet (No Quantum)    & \ul{0.89}    & \ul{0.89}    & \ul{0.89}    & \ul{0.89}    \\
                          & QCL-MixNet (No Mixup)      & 0.88          & 0.88          & 0.88          & 0.88          \\
                          & QCL-MixNet (No Attention)  & 0.87          & 0.88          & 0.87          & 0.87          \\ \cmidrule(l){2-6} 
                          & Diff                       & \textcolor{mygreen}{+0.06}         & \textcolor{mygreen}{+0.06}         & \textcolor{mygreen}{+0.06}         & \textcolor{mygreen}{+0.06}         \\ \midrule
\multirow{5}{*}{glass}    & \textbf{QCL-MixNet (Full)} & \textbf{0.79} & \textbf{0.66} & \textbf{0.66} & \textbf{0.61} \\
                          & QCL-MixNet (No Quantum)    & \ul{0.6}     & 0.33          & 0.4           & 0.34          \\
                          & QCL-MixNet (No Mixup)      & 0.56          & \ul{0.51}    & \ul{0.42}    & \ul{0.39}    \\
                          & QCL-MixNet (No Attention)  & 0.42          & 0.22          & 0.31          & 0.25          \\ \cmidrule(l){2-6} 
                          & Diff                       & \textcolor{mygreen}{+0.19}         & \textcolor{mygreen}{+0.15}         & \textcolor{mygreen}{+0.24}         & \textcolor{mygreen}{+0.22}         \\ \bottomrule[1.5pt]
\end{longtable}%
\end{sc}
}

\section{Results and Discussion}
\label{sec:results}

\subsection{Performance Comparison}
Our proposed QCL-MixNet demonstrates consistently superior performance across a diverse set of 7 binary tabular datasets, as detailed in Table~\ref{tab:2}.Notably, on datasets like~\textit{arrhythmia},~\textit{pen\_digits},~\textit{satimage} and,~\textit{optical\_digits}, QCL-MixNet secures the best maF1 values. This improvement is visualized in Figure~\ref{fig:improvement-binary}, which shows the absolute gain in Macro F1 score compared to the best-performing baseline for each dataset. This strong performance, particularly in maF1, which is crucial for imbalanced scenarios. Traditional DL models, including advanced architectures like BiLSTMs or ResNets, often falter on imbalanced tabular data due to overfitting to the majority class or inadequate representation of minority classes. QCL-MixNet demonstrably overcomes these limitations in most cases. For instance, on~\textit{satimage}, QCL-MixNet achieves an maF1 of 0.84, surpassing the strongest competitors (second highest maF1: 0.82). Here, the `Diff' row, which quantifies QCL-MixNet's improvement over the best single model from any other category, shows a +0.02 maF1 gain, highlighting its superior handling of minority classes. This is likely due to its dynamic mixup, which creates more informative synthetic samples, and its QI feature learning, which improves separability.

This robust performance extends to multi-class classification, where QCL-MixNet shows leading performance across 11 challenging tabular datasets (see Table~\ref{tab:3}), frequently outperforming all other model categories (ML, DL, and GNNs). It frequently outperforms all other model categories (ML, DL, and GNNs). The performance gains on these multiclass datasets are summarized in Figure~\ref{fig:improvement-multiclass}, which shows that QCL-MixNet outperforms the top baseline on 8 of the 11 datasets. For instance, on highly imbalanced datasets like~\textit{vehicle},~\textit{wine-quality-red}, and~\textit{glass}, QCL-MixNet demonstrates substantial improvements in accuracy, with `Diff' values of +0.09, +0.04, and +0.09, respectively, over the best competing model from other categories. These results underscore the efficacy of its dynamic mixup strategy. This strategy is hypothesized to generate synthetic samples adaptively, particularly in underrepresented regions of the feature space defined by class imbalance and feature overlap. Such targeted augmentation is crucial for enabling better separation of multiple minority classes, a common challenge where standard oversampling or augmentation methods often falter in multi-class settings.

The results indicate that QCL-MixNet surpasses all other methods in accuracy across 14 out of 18 binary and multiclass datasets, and leads in maP (11 datasets), maR (9 datasets), and maF1 (12 datasets). While strong baselines like SVM (SMOTE) excel on specific datasets (e.g.,~\textit{one-hundred-plants-texture}), QCL-MixNet demonstrates superior consistency, significantly outperforming these methods across a wider range of datasets, especially those with severe class imbalances or complex structures (e.g.,~\textit{lymph},~\textit{wine-quality-white}). This highlights QCL-MixNet's robust and generalizable performance in imbalanced learning across diverse datasets, rather than focusing on dataset-specific efficiency, driven by its innovative integration of quantum-informed principles with dynamic mixup augmentation, which enables more effective regularization and feature learning for imbalanced tabular classification.

To further illustrate these findings, Figure~\ref{fig:improvement_overall} provides a direct, per-dataset comparison of QCL-MixNet against the strongest baseline. The green bars indicate an improvement in maF1, with particularly large gains on datasets like \textit{arrhythmia} and \textit{vehicle}. The few red bars in Figure~\ref{fig:improvement-multiclass} correspond to the handful of cases where a baseline model retained a slight edge. Complementing this, Figure~\ref{fig:cross_overall} contextualizes the overall performance landscape. The box plots show the distribution of maF1 scores for each model across all binary (Figure~\ref{fig:cross-binary}) and multiclass (Figure~\ref{fig:cross-multi}) datasets, respectively. These plots reveal the high variance and lower median performance of some traditional models, establishing the highly competitive nature of baselines like ResNet against which QCL-MixNet's improvements are measured. As illustrated in Figure~\ref{fig:tsne_main}, the t-SNE embeddings of QCL-MixNet’s prediction space reveal consistent and interpretable clustering behavior across both binary and multiclass classification tasks. In the binary datasets (Figure~\ref{tsne-binary}), the model exhibits distinct and compact clusters for the two classes, with minimal overlap, especially evident in datasets such as \textit{pen\_digits}, \textit{isolet}, and \textit{optical\_digits}, suggesting that QCL-MixNet effectively captures discriminative representations even in low-dimensional latent spaces. On the multiclass side (Figure~\ref{tsne-multiclass}), datasets like \textit{letter}, \textit{har}, and \textit{one-hundred-plants-texture} display well-separated groupings aligned with the underlying class labels, demonstrating the model’s capacity to preserve inter-class structure. Notably, even for challenging datasets such as \textit{abalone} and \textit{glass}, where classes are inherently less separable, the embeddings retain a coherent spatial distribution.

\subsection{Ablation Studies}
To dissect the contributions of key architectural modules (QE Layers, Dynamic Mixup, and Attention), we conducted systematic ablation experiments. Figure~\ref{fig:ablation_overall} provides a high-level visual summary of these findings by plotting the mean performance across datasets. Detailed per-dataset results are shown in Table~\ref{tab:ablation-binary} for binary datasets and Table~\ref{tab:ablation-multi} for multi-class datasets. 
Among all components, the removal of Dynamic Mixup led to the most significant degradation, with an average maF1 drop of approximately 24.5\% across the 18 datasets. The absence of mixup severely hampers generalization in datasets with either many classes or high intra-class variability (e.g.,~\textit{one-hundred-plants-texture},~\textit{glass}). Mixup appears to lessen overfitting and reduce decision boundary sharpness, promoting interpolation between samples, which is especially beneficial in low-data regimes such as~\textit{glass} with only 214 instances. The attention mechanism also proved highly influential, its exclusion resulting in an average maF1 decrease of about 23.2\%; its importance is highlighted on datasets with high feature dimensionality (e.g.,~\textit{har} with 562 features,~\textit{mineral} with 140 features). This suggests that the attention module is crucial for feature selection and localization, enabling the model to suppress irrelevant or noisy features, especially when feature redundancy is high or sample diversity is sparse. 

The QE Layers, while having a more varied impact across datasets, still demonstrated substantial overall contribution, with their removal causing an average maF1 reduction of roughly 11.5\%, most notably in high-imbalance or high-dimensional datasets (e.g.,~\textit{vehicle},~\textit{arrhythmia}). This supports the hypothesis that these QI layers introduce beneficial inductive biases that improve representation expressiveness, particularly under data sparsity. Moreover, QCL-MixNet exhibits high stability across varying data sizes, from compact datasets (e.g.,~\textit{ecoli} with 336 instances) to large ones (e.g.,~\textit{letter} with 20,000 instances), indicating architectural scalability. Ultimately, the consistent superiority of the full QCL-MixNet model over its ablated variants emphasizes that its components—improved representation learning, strategic data augmentation, and refined feature weighting—work synergistically to drive robust performance. This synergistic effect is visualized in Figure~\ref{fig:ablation_overall}, where the blue bar representing the full model consistently stands taller than the bars for any of the ablated versions across all test scenarios.

\subsection{Statistical Comparison Using the Friedman Test}
To statistically compare the performance of multiple classification models across multiple datasets, we employed the \textit{Friedman test}, a widely accepted non-parametric procedure suitable for multiple comparisons under repeated measures. This method is particularly suitable for evaluating algorithms across a common set of datasets, where the assumptions of parametric tests, such as normality and homoscedasticity, are unlikely to hold. To rigorously compare the performance of multiple classifiers across datasets and metrics, we applied the Friedman test in two complementary ways: (i) on a per-metric basis across datasets, and (ii) globally, averaging the ranks across metrics. These tests account for non-normality and repeated measures, providing a robust statistical basis for evaluating model differences.

Let us denote by $k$ the number of models being compared, and by $N$ the number of repeated measures, which can correspond to either the number of datasets (in the per-metric Friedman test) or the number of evaluation metrics (in the global Friedman test). For each repetition $i \in \{1, \dots, N\}$, the $k$ models are evaluated using a common criterion (e.g., classification accuracy or another performance measure). These scores are then converted into ranks $R_{ij} \in \{1, \dots, k\}$, where $R_{ij}$ denotes the rank of the $j$-th model under the $i$-th repetition. The best-performing model receives rank 1, the second-best receives rank 2, and so on. In the event of ties, average ranks are assigned to the tied models.

The null hypothesis ($H_0$) of the Friedman test posits that all models perform equally well in expectation, implying that their mean ranks are equal:
\begin{equation}
H_0: \mathbb{E}[R_{1}] = \mathbb{E}[R_{2}] = \dots = \mathbb{E}[R_{k}]
\end{equation}
To evaluate this hypothesis, the Friedman test statistic $\chi_F^2$ is computed as:
\begin{equation}\label{eq:friedman}
\chi_F^2 = \frac{12N}{k(k+1)} \sum_{j=1}^k \bar{R}_j^2 - 3N(k+1)
\end{equation}
where $\bar{R}_j = \frac{1}{N} \sum_{i=1}^N R_{ij}$ is the average rank of the $j$-th model across all $N$ repeated measures (either datasets or evaluation metrics). In settings where ties may occur, such as when ranking models on individual datasets, it is necessary to adjust for the reduced variance in ranks caused by tied values. A correction factor $c \in (0, 1]$ is introduced:
\begin{equation}
c = 1 - \frac{T}{Nk(k^2 - 1)}
\end{equation}
where $T$ is the total tie correction term, obtained by summing $t(t^2 - 1)$ over all tie groups of size $t$ across the repeated measures. The corrected test statistic is then:
\begin{equation}
\chi_{F,\text{ corrected}}^2 = \frac{\chi_F^2}{c}
\end{equation}

In our analysis, this correction was applied only in the \textit{per-metric Friedman tests}, where ties were possible within individual datasets. For the \textit{global Friedman test} based on aggregated average ranks across metrics, no correction was necessary. Assuming a sufficiently large number of repetitions ($N > 10$), the Friedman statistic (corrected or uncorrected) asymptotically follows a $\chi^2$ distribution with $k-1$ degrees of freedom ($df$). A small $p$-value indicates that the $H_0$ can be rejected, suggesting statistically significant differences in model performance.

\begin{table*}[hb!]
\caption{Average Ranks of Models Across All Datasets and Friedman Test Outcomes for Each Evaluation Metric}
\label{tab:stat-test}
\resizebox{\textwidth}{!}{%
\begin{tabular}{@{}lcccccccccccccccccccccccc@{}}
\toprule[1.5pt]
\textbf{Metrics} &
  \textbf{XGBoost} &
  \textbf{Balanced RF} &
  \textbf{SVM (SMOTE)} &
  \textbf{MLP} &
  \textbf{ResNet} &
  \textbf{DT} &
  \textbf{RF} &
  \textbf{GB} &
  \textbf{LR} &
  \textbf{KNN} &
  \textbf{BiGRU} &
  \textbf{GRU} &
  \textbf{LSTM} &
  \textbf{BiLSTM} &
  \textbf{CNN} &
  \textbf{GCN} &
  \textbf{GraphSAGE} &
  \textbf{GAT} &
  \textbf{GIN} &
  \textbf{JKNet} &
  \textbf{QCL-MixNet} &
  \textbf{Friedman Statistic ($\chi^2$)} &
  \textbf{p-value} &
  \textbf{Decision on  $H_0$} \\ \midrule[1pt]
Accuracy &
  9.9 &
  14 &
  8.2 &
  11.6 &
  10.4 &
  15.6 &
  7.1 &
  9.6 &
  15.5 &
  6.9 &
  9 &
  10.8 &
  12.2 &
  10 &
  7.4 &
  15.9 &
  9.1 &
  14.8 &
  13.9 &
  14.1 &
  3.8 &
  99.42 &
  1.60$\times 10^{-12}$ &
  Reject \\
maR &
  11.7 &
  8.4 &
  2.8 &
  10.9 &
  10.9 &
  16 &
  7.9 &
  9.2 &
  17.5 &
  5.9 &
  10.3 &
  10.4 &
  13.6 &
  10.9 &
  9.7 &
  16.7 &
  10.6 &
  14.2 &
  14.9 &
  14.9 &
  3.6 &
  156.58 &
  3.43$\times 10^{-23}$ &
  Reject \\
maP &
  12.7 &
  10.1 &
  3.6 &
  11.2 &
  10.6 &
  16.8 &
  6.5 &
  8.5 &
  17.4 &
  5.1 &
  8.5 &
  10.1 &
  12.8 &
  10.4 &
  9.1 &
  15.9 &
  12.1 &
  14.5 &
  15.5 &
  15.9 &
  3.6 &
  174.84 &
  9.91$\times 10^{-27}$ &
  Reject \\
maF1 &
  11.5 &
  10 &
  2.8 &
  11.1 &
  10.6 &
  16 &
  7.2 &
  8.6 &
  17.8 &
  4.9 &
  10.3 &
  10.4 &
  13.7 &
  11 &
  8.8 &
  17.5 &
  11.2 &
  14.8 &
  14.6 &
  14.8 &
  3.2 &
  174.91 &
  9.57$\times 10^{-27}$ &
  Reject \\ \midrule
Average Rank &
  11.45 &
  10.625 &
  4.35 &
  11.2 &
  10.625 &
  15.85 &
  7.175 &
  9.975 &
  17.05 &
  5.7 &
  9.525 &
  10.425 &
  13.075 &
  10.575 &
  8.75 &
  16.5 &
  10.75 &
  14.575 &
  14.725 &
  14.925 &
  \textbf{3.55} &
  29.68 & 
  1.44$\times 10^{-2}$ &
  Reject \\ \bottomrule[1.5pt]
\end{tabular}%
}
\end{table*}

\subsubsection{Per-Metric Friedman Tests Across Datasets}
To assess whether the observed performance differences among the $k=21$ classification models across the $n = 18$ datasets are statistically significant for each evaluation metric, we conducted individual Friedman tests per metric. The test computes ranks for each model within each dataset, then evaluates whether the mean ranks differ significantly under the $H_0$ that all models perform equally well. Table~\ref{tab:stat-test} presents the average ranks of all models along with the Friedman test outcomes. In all cases, the Friedman test strongly rejects $H_0$ ($p \ll 0.05$), indicating statistically significant differences in performance rankings across models. Among the models, the SVM (SMOTE) variant achieved the best (lowest) average rank in maR (2.8) and maF1 (2.8), while the proposed QCL-MixNet achieved the best performance in accuracy (3.8) and maP (3.6). To resolve the ambiguity arising from multiple models excelling in different metrics, we performed a global Friedman test over the average ranks aggregated across all four metrics. This allows us to assess the models’ overall consistency and performance in a unified statistical framework.

\subsubsection{Global Friedman Test on Aggregated Ranks Across Metrics}
To complement the per-metric Friedman tests, we also performed a global Friedman test (bottom row of Table~\ref{tab:stat-test}) on the average ranks of each model, aggregated across the four evaluation metrics. This approach offers a more consolidated view of model performance by considering multiple evaluation dimensions simultaneously. Let $k = 21$ be the number of classifiers and $N = 4$ be the number of metrics used as repeated measures. The average ranks $\bar{R}_j$ for each model $j \in {1, \dots, k}$ were computed by averaging its rank across the four metrics. Plugging in the values in Equation~\ref{eq:friedman}, we obtained: $\chi_F^2 = \frac{12 \times 4}{21 \times 22} \sum_{j=1}^{21} \bar{R}_j^2 - 3 \times 4 \times 22 = 29.68$. Since $N$ is small, we applied the Iman-Davenport correction to obtain an $F$-distributed statistic: $F_F = \frac{(N-1)\chi_F^2}{N(k-1) - \chi_F^2} = \frac{(4 - 1) \times 29.68}{4 \times (21 - 1) - 29.68} = \frac{89.04}{50.32} \approx 1.77$. The $df$ for this test are: $df_1 = k - 1 = 20$ and $df_2 = (k - 1)(N - 1) = 20 \times 3 = 60$. The critical value at $\alpha = 0.05$ from the $F$-distribution is $F_{0.05}(20, 60) \approx 1.748$. Since $F_F = 1.77 > 1.748$, we reject the $H_0$, which posits that all models perform equally in expectation. This result confirms that the observed rank differences across models are statistically significant even when aggregating across multiple evaluation dimensions. QCL-MixNet, with the lowest global mean rank of 3.55, again emerged as the top-performing model.

% \section{Discussion}

% Why quantum-inspired features improve separability in minority classes.

% Practical implications for real-world expert systems.

% Limitations: Compute requirements for mixup; sensitivity to kNN hyperparameters.

\section{Conclusion}
\label{sec:conclude}
In this paper, we introduced QCL-MixNet, a novel framework designed to address the critical challenge of class imbalance in tabular data. Our experiments on 18 binary and multi-class datasets with different imbalance ratios demonstrate that QCL-MixNet consistently outperforms leading ML, DL, and GNN models. This superiority is validated by the Friedman test. Furthermore, Ablation studies show that the complete QCL-MixNet consistently outperforms its ablated variants. The key to QCL-MixNet’s effectiveness, especially in handling minority classes, is its ability to create a more separable feature space. Its QI layers use structured nonlinear transformations to capture richer feature interactions, helping to differentiate rare instances that are often dominated by the majority class. This enriched feature representation, combined with a targeted augmentation strategy that generates realistic synthetic samples, enables the model to learn more accurate decision boundaries. The success of our method in improving the identification of minority classes has direct implications for real-world expert systems in domains like medical diagnosis, fraud detection, and industrial fault prediction, where the cost of missing a rare event is extremely high. By improving the identification of minority classes, our model can help build more reliable and equitable AI-driven decision-making tools. 

While our QCL-MixNet shows promising results, there are a few areas for future improvement. First, our kNN-guided dynamic mixup, though effective in generating meaningful augmentations, requires additional computation to find neighbors, and its performance can be sensitive to the number of \textit{k-neighbors}. Future work could explore efficient approximate kNN algorithms or develop automated, perhaps learnable, strategies for optimizing the mixup process, including neighbor selection and interpolation ratios. Second, while our current work successfully integrates quantum-inspired DL components to enhance feature representation, a more advanced future direction, which was beyond the scope of this particular study, could be the exploration of true quantum-classical hybrid architectures~\citep{jahin_qamplifynet_2023}. Such systems might integrate actual quantum computational elements (e.g., variational quantum circuits for feature mapping) with classical networks, which is a promising direction for future research. Lastly, the sophisticated hybrid loss function, while powerful, involves several interacting hyperparameters (e.g., component weights, temperature, margins), which can complicate model tuning. Developing automated techniques to find the best settings for these would make the model easier to apply widely.
%Recap of contributions and empirical validation.Why quantum-inspired features improve separability in minority classes.Practical implications for real-world expert systems.Limitations: Compute requirements for mixup; sensitivity to kNN hyperparameters.Directions: Quantum-classical fusion architectures, automated mixup optimization.

\section*{Data Availability Statement}
All 18 datasets used in this study are publicly available and were obtained from the UCI Machine Learning Repository, OpenML, and Kaggle. Detailed information and links are provided in Table~\ref{tab:1}.

\section*{Funding Sources}
This research did not receive any specific grant from funding agencies in the public, commercial, or not-for-profit sectors.

% \appendix
% \section{My Appendix}
% Appendix sections are coded under \verb+\appendix+.

% \verb+\printcredits+ command is used after appendix sections to list 
% author credit taxonomy contribution roles tagged using \verb+\credit+ 
% in frontmatter.

\printcredits

% \bibliographystyle{cas-model2-names}
\bibliographystyle{apalike}
\bibliography{main}


% Author biography without author photo.
% Author biography. Author biography. Author biography.
% Author biography. Author biography. Author biography.
% Author biography. Author biography. Author biography.
% Author biography. Author biography. Author biography.
% Author biography. Author biography. Author biography.
% Author biography. Author biography. Author biography.
% Author biography. Author biography. Author biography.
% Author biography. Author biography. Author biography.
% Author biography. Author biography. Author biography.

% \bio{figs/pic1}
% Author biography with author photo.
% Author biography. Author biography. Author biography.
% Author biography. Author biography. Author biography.
% Author biography. Author biography. Author biography.
% Author biography. Author biography. Author biography.
% Author biography. Author biography. Author biography.
% Author biography. Author biography. Author biography.
% Author biography. Author biography. Author biography.
% Author biography. Author biography. Author biography.
% Author biography. Author biography. Author biography.

% \bio{figs/pic1}
% Author biography with author photo.
% Author biography. Author biography. Author biography.
% Author biography. Author biography. Author biography.
% Author biography. Author biography. Author biography.
% Author biography. Author biography. Author biography.

\end{document}